\documentclass[11pt]{article}

\usepackage[margin=1in]{geometry}
\usepackage[numbers,compress]{natbib}
\usepackage[utf8]{inputenc}
\usepackage[T1]{fontenc}
\usepackage{hyperref}
\usepackage{url}
\usepackage{booktabs}
\usepackage{siunitx}
\usepackage{amsfonts}
\usepackage{nicefrac}
\usepackage{microtype}
\usepackage[table]{xcolor}
\usepackage{amsmath, amssymb, amsthm}
\usepackage{mathtools}
\usepackage{enumitem}
\usepackage{graphicx}
\usepackage{subcaption}
\usepackage{tabularx}
\usepackage{array}
\usepackage{multirow}
\usepackage{placeins}
\usepackage{float}
\usepackage{algorithm}
\usepackage{algpseudocode}
\usepackage{titletoc}
\usepackage[most]{tcolorbox}

\graphicspath{{../../../m2s/results/mnist/}{../results/assets/figures/}}

\definecolor{appthmgray}{gray}{0.96}
\definecolor{appthmborder}{gray}{0.72}
\newtcolorbox{appthmbox}[1]{
  enhanced,
  breakable,
  colback=appthmgray,
  colframe=appthmborder,
  boxrule=0.4pt,
  arc=0pt,
  left=6pt,
  right=6pt,
  top=6pt,
  bottom=6pt,
  before skip=8pt,
  after skip=8pt,
  before upper={\textbf{#1.}\enspace\itshape}
}
\newtcolorbox{appthmframe}{
  enhanced,
  breakable,
  colback=appthmgray,
  colframe=appthmborder,
  boxrule=0.4pt,
  arc=0pt,
  left=6pt,
  right=6pt,
  top=4pt,
  bottom=4pt,
  before skip=8pt,
  after skip=8pt
}

\definecolor{sampleboxblue}{HTML}{EEF5FA}
\definecolor{sampleboxborder}{HTML}{5B8DB8}
\newtcolorbox{owtsample}[1]{
  enhanced,
  breakable,
  colback=sampleboxblue,
  colframe=sampleboxborder,
  coltitle=black,
  title={#1},
  fonttitle=\bfseries\small,
  fontupper=\small,
  boxrule=0.5pt,
  arc=1.5mm,
  left=8pt,
  right=8pt,
  top=6pt,
  bottom=6pt,
  before skip=2pt,
  after skip=2pt,
  before upper={\setlength{\parindent}{1em}\setlength{\parskip}{0.35em}}
}

\theoremstyle{plain}
\newtheorem{theorem}{Theorem}[section]

\newtheorem{proposition}[theorem]{Proposition}

\theoremstyle{definition}
\newtheorem{definition}[theorem]{Definition}
\newtheorem{assumption}[theorem]{Assumption}

\theoremstyle{remark}

\newcommand{\R}{\mathbb{R}}
\newcommand{\E}{\mathbb{E}}
\newcommand{\1}{\mathbf{1}}
\newcommand{\ind}{\mathbf{1}}
\newcommand{\simplex}{\Delta^{K-1}}
\newcommand{\pdata}{p_{\mathrm{data}}}

\newcommand{\MtwoS}{\mathrm{M2S}}
\newcommand{\CifarMtwoSFID}{28.09}
\newcommand{\CifarSEDDFID}{42.83}
\newcommand{\CifarFIDDelta}{14.74}
\newcommand{\CifarFIDReductionPct}{34.4\%}
\DeclareMathOperator{\Prb}{Pr}
\DeclareMathOperator{\softmax}{softmax}

\title{Mean-to-Score Discrete Diffusion:\\ Posterior-Mean Denoisers for Score Entropy}

\author{%
  \textbf{Jingyuan Li}$^{2,3,4}$\thanks{Equal contribution.}
  \quad
  \textbf{Xiaoyi Jiang}$^{1,4}$\footnotemark[1]
  \quad
  \textbf{Yixuan Jiang}$^{1,4}$
  \quad
  \textbf{Wei Liu}$^{3}$ \\[0.35ex]
  \quad
  \textbf{Yi Zhu}$^{1,2,4}$
  \quad
  \textbf{Zuoqiang Shi}$^{1,2,4}$
  \quad
  \textbf{Pipi Hu}$^{2,4}$\thanks{Corresponding author.} \\[0.5ex]
  $^{1}$Tsinghua University
  \quad
  $^{2}$Beijing Institute of Mathematical Sciences and Applications \\
  $^{3}$Wuhan University
  \quad
  $^{4}$MathonAI
}
\date{}

\begin{document}

\maketitle

\begin{abstract}
Score Entropy Discrete Diffusion (SEDD) parameterizes discrete reverse processes
with unconstrained positive score ratios. Although positivity ensures
nonnegative reverse jump rates, it does not
ensure Bayes realizability: at a fixed noisy state, all candidate score ratios
must arise jointly from a single clean-token posterior under the forward
kernel. We show that the score-entropy loss has the correct population
optimum but does not enforce Bayes realizability away from it. In a
trained pure-uniform SEDD checkpoint, roughly one quarter of complete score vectors
violate the coordinate box, while more than half satisfy every
coordinate bound but remain materially incompatible with any valid clean-token
posterior. Although the corresponding continuous-time reverse jump rates remain
nonnegative, these violations can induce negative pre-normalization weights in
the finite-step sampler update. Projecting the checkpoint's raw scores onto the
bridge polytope removes all observed negative weights and lowers external generative
PPL from $203.6$ to $175.1$ without changing the sampler. To enforce Bayes
realizability by construction rather than through post-hoc projection, we
introduce \emph{mean-to-score} (M2S): the network predicts a clean-token
posterior mean and converts it to the score through an exact kernel-dependent
linear map. The map
applies to any known coordinate-wise continuous-time
Markov chain (CTMC) satisfying a mild support condition. For uniform corruption,
it maps the probability simplex onto the bridge polytope; for absorbing-mask
corruption, the
resulting objective recovers MD4 exactly.
In a controlled 28.4M-parameter CIFAR-10 comparison, M2S lowers test BPD
from $3.173$ to $3.129$ and FID-50k from $\CifarSEDDFID$ to
$\CifarMtwoSFID$.
A 170M-parameter M2S
model trained on approximately 262B OpenWebText token slots outperforms the
evaluated pure-uniform SEDD, GIDD, and Neural CTMC checkpoints at every tested
sampling budget, reaching generative PPL $143.3$ at 128 steps compared with
$183.6$ for the strongest pure-uniform baseline.
\end{abstract}

\begin{figure}[!ht]
  \centering
  \includegraphics[width=0.98\linewidth]{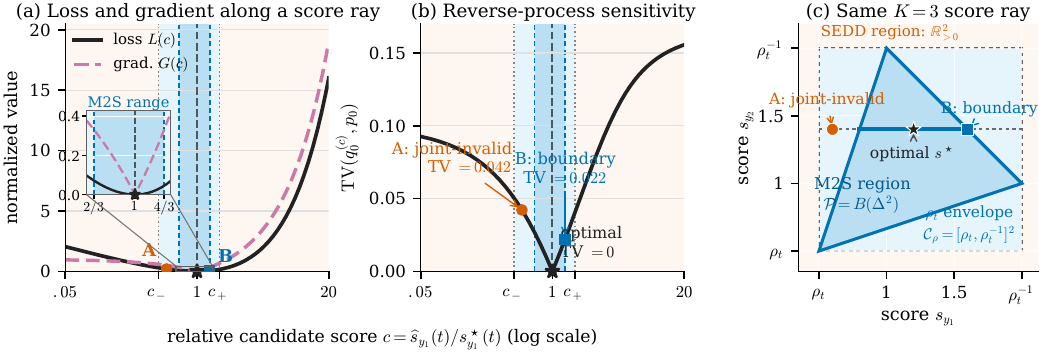}
  \caption{\textbf{M2S occupies the Bayes-realizable subset of the scalar envelope, whereas SEDD can output any positive score.}
  In a $K=3$ uniform toy, panels~(a)--(b) vary one score along
  $\widehat s(c)=(c s_{y_1}^\star,s_{y_2}^\star)$ around the shared optimum
  $c=1$. The labels $c_-$ and $c_+$ mark the limits imposed by the
  $\rho_t$-based coordinate envelope, while the inner dark-blue band is the
  smaller Bayes-realizable subset enforced by M2S and enlarged in the inset.
  Panel~(a) shows the score-entropy excess and log-score gradient; panel~(b)
  shows the terminal reverse-process TV error; and panel~(c) shows the joint
  score region. Point A passes every scalar bound but admits no valid joint
  posterior, while point B lies on the M2S boundary.}
  \label{fig:score_constraint_toy}
\end{figure}

\section{Introduction}
\label{sec:intro}

Discrete diffusion models generate finite-valued data by reversing a
continuous-time Markov chain (CTMC) that progressively corrupts a
sample~\citep{austin2021structured,campbell2022ctmc}. We study pure-uniform
corruption, which treats vocabulary states symmetrically without an absorbing
mask token. For site $i$, let $X_0^i$ denote the clean token and
$\mu_i^\star(x,t):=\Prb(X_0^i=\cdot\mid X_t=x)\in\Delta^{K-1}$ its posterior
given the noisy sequence $X_t=x$, where $\Delta^{K-1}$ is the probability
simplex over a vocabulary of size $K$. Unlike absorbing-mask training, which reduces to a weighted
clean-token cross-entropy on marked positions, uniform corruption does not
reveal which positions changed and requires the reverse model to coordinate
transitions across the full vocabulary~\citep{sahoo2024simple,shi2024simplified}.

MD4~\citep{shi2024simplified} already highlights a closely related
parameterization issue for absorbing-mask diffusion: a freely parameterized
score need not be induced by the conditional clean-token mean under the forward
process. MD4 therefore predicts this mean and constructs the masked reverse
model from it. We ask a vector-level question: whether one valid clean-token
posterior induces the complete concrete-score vector predicted by SEDD. We call
this property \emph{Bayes realizability}. It reduces to MD4's score--mean
constraint for an absorbing mask, but under pure-uniform corruption it couples
all $K-1$ scores.

In score-based CTMC models, each reverse jump rate uses a concrete score ratio
$p_t(x^{i\to y})/p_t(x)$. SEDD~\citep{lou2024discrete} predicts each ratio with
an exponentiated network output, guaranteeing positive scores and nonnegative
continuous-time reverse jump rates. \textbf{Positivity alone, however, does not
ensure that all candidate score ratios arise from one clean-token posterior.}

Formally, at noisy state $x$ and site $i$ with $k=x^i$, a candidate-score vector
$s=(s_y)_{y\neq k}$ satisfies Bayes realizability if $s=B_{t,k}\mu$ for some posterior
$\mu\in\Delta^{K-1}$. Under uniform corruption, any score induced by a
clean-token posterior must lie in the coordinate box
$\mathcal C_{t,k}:=[\rho_t,\rho_t^{-1}]^{K-1}$. This restriction concerns
Bayes realizability, not
requirement for nonnegative continuous-time reverse rates, which only require
$s_y>0$. The coordinate box is necessary but not sufficient for Bayes
realizability: the complete score vector must lie in the strictly smaller
bridge polytope $\mathcal P_{t,k}=B_{t,k}(\Delta^{K-1})\subsetneq\mathcal C_{t,k}$.
Thus, constraining a SEDD head to positive
outputs is insufficient to ensure Bayes realizability: its output may lie
outside both the bridge polytope and even the coordinate box. Score
entropy therefore has the correct population optimum but does not enforce Bayes
realizability away from that optimum. Figure~\ref{fig:score_constraint_toy} visualizes the distinction for $K=3$.
The coordinate-box interval strictly contains its intersection with the
bridge polytope, and panel~(c) shows
$\mathcal P_{t,k}\subsetneq\mathcal C_{t,k}\subsetneq\R_{>0}^2$. Point A
passes every coordinate bound but has a negative signed inverse component.

Failure of Bayes realizability has measurable consequences at
language scale. In a pure-uniform SEDD checkpoint, $6.69\%$ of $3.29$ billion
audited candidate scores
leave the coordinate box. To test Bayes realizability during generation, we
run this checkpoint with SEDD's released
sampler on 128 sequences and audit the unmodified score vector at every
position and sampling step. Of these vectors, $25.02\%$ contain a coordinate
violation, while another $56.23\%$ pass every scalar check but admit no valid
joint posterior. During these runs, the sampler update encounters negative
pre-normalization weights and operationally samples from their positive part.
To isolate the effect, we project
the scores onto $\mathcal P_{t,k}$ for 1,024 paired sequences while holding
the time grid, initial states, random-number stream, final denoising, and sampler
code fixed. Projection removes all observed negative weights and lowers external
generative PPL from $203.60$ to $175.07$.

Rather than project scores at inference time, we enforce Bayes realizability in the
model. For any known coordinate-wise forward kernel satisfying a mild support
condition, the one-site clean-token posterior recovers every concrete score:
\begin{equation}
\label{eq:intro_bridge}
  s_i^\star(x,t;y)
  =\E\!\left[
    \frac{P_t^{(i)}(y\mid X_0^i)}{P_t^{(i)}(x^i\mid X_0^i)}
    \,\middle|\, X_t=x
  \right],
\end{equation}
where the integrand depends only on the clean token at site $i$. Mean-to-score
(M2S) predicts $\mu_\theta^i(x_t,t)=\softmax(x_\theta^i(x_t,t))$ and applies the
known linear map $B_t$ to obtain the score. Because $\mu_\theta^i$ is a
distribution, M2S enforces Bayes realizability. It changes the admissible off-optimum
geometry, not the population target.

The bridge uses only a one-site posterior, not a posterior over the full
sequence. For uniform corruption it is injective on the simplex, has a closed
form, and evaluates all scores in $O(K)$ time per site. For absorbing-mask
corruption, the same construction recovers the weighted clean-token
cross-entropy of MD4 exactly.
We also show that the conditional score-entropy risk is uniquely minimized at
the true score (Proposition~\ref{prop:consistency}) and derive the rank condition
$\operatorname{rank}([B_+^\top,\1]^\top)=K$ under which this optimum uniquely recovers
the clean-token posterior (Theorem~\ref{thm:mu_optimal}); pure-uniform
corruption satisfies this condition.

Our experiments ask whether enforcing Bayes realizability matters for image and
language generation. On 256-state MNIST, we compare M2S with SEDD: the two
models share the architecture,
forward process, objective, training budget, and sampler, and M2S improves FID
from $126.1\pm0.4$ to $71.1\pm4.3$.
On CIFAR-10, an augmentation-free 28.4M-parameter M2S checkpoint reaches
a test BPD upper bound of $3.129$, compared with $3.173$ for SEDD under
the same evaluation protocol; under the same 256-step Euler sampler,
M2S also lowers FID-50k from $\CifarSEDDFID$ to $\CifarMtwoSFID$. On
OpenWebText~\citep{Gokaslan2019OpenWeb}, the best configuration of
our 169.9M-parameter M2S model reaches generative PPL $143.3$ at 128 steps and
outperforms the evaluated pure-uniform SEDD~\citep{lou2024discrete},
GIDD~\citep{vonrutte2025gidd}, and Neural CTMC~\citep{li2026neuralctmc}
baselines at all tested budgets. Figure~\ref{fig:owt_checkpoint_audit} further
compares M2S with pure-uniform SEDD through Bayes realizability and
score-entropy audits, and adds a CIFAR-10 optimization trace under identical
compute.
Keeping the SEDD checkpoint fixed, we also project its
scores onto $\mathcal P_{t,k}$ during sampling using the simplex-constrained
Euclidean projection in Eq.~\eqref{eq:projected_posterior}, lowering external
GenPPL from $203.60$ to $175.07$
(Section~\ref{sec:same_checkpoint_score_repair}).

\begin{figure}[!t]
  \centering
  \includegraphics[width=0.93\linewidth]{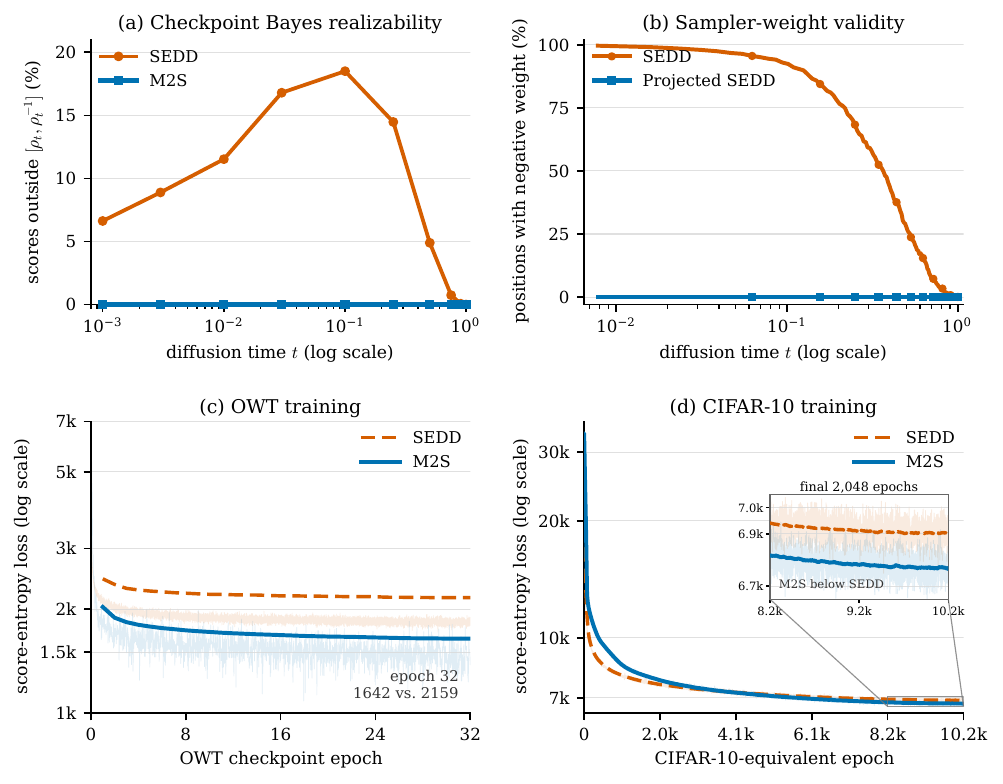}
  \caption{\textbf{M2S enforces Bayes realizability and achieves lower final
  loss on both text and image generation.}
  Panels~(a)--(b) audit score realizability and sampler-weight validity.
  Panels~(c)--(d) compare the score-entropy loss of M2S and SEDD on
  OpenWebText and CIFAR-10, respectively.}
  \label{fig:owt_checkpoint_audit}
\end{figure}

\paragraph{Contributions.}
\begin{enumerate}[leftmargin=*]
  \item \textbf{We introduce joint Bayes realizability for discrete score
  vectors.} A complete score vector satisfies Bayes realizability exactly when
  it is induced by one valid clean-token posterior.
  Positivity and even coordinate-wise feasibility do not suffice. Under
  uniform corruption, such vectors form a strict bridge
  polytope $\mathcal P_{t,k}\subsetneq\mathcal C_{t,k}$, and score entropy does
  not enforce this constraint away from its population optimum.

  \item \textbf{We derive M2S to enforce Bayes realizability.}
  M2S maps a one-site clean-token posterior to all concrete scores through an
  exact kernel-dependent linear bridge. We establish score consistency and
  posterior recovery, derive the uniform-kernel form, and recover the MD4
  objective under absorbing-mask corruption.

  \item \textbf{We show empirically that Bayes realizability matters.}
  For a fixed pure-uniform SEDD checkpoint, simplex-constrained Euclidean
  projection removes negative sampler weights and lowers generative PPL from
  $203.6$ to $175.1$. On MNIST and CIFAR-10, M2S improves FID over SEDD under
  identical settings, while a 170M-parameter OpenWebText model outperforms the
  evaluated pure-uniform SEDD, GIDD, and Neural CTMC baselines at every tested
  sampling budget.
\end{enumerate}

\FloatBarrier
\section{Related Work}
\label{sec:related}

\paragraph{Discrete diffusion.}
Diffusion probabilistic models originate from progressively corrupting Markov
chains \citep{sohl2015deep}. Discrete variants use multinomial, structured, or
absorbing transition kernels
\citep{hoogeboom2021argmax,austin2021structured,bondtaylor2022unleashing,
he2023diffusionbert,zheng2024reparameterized}. Continuous-time formulations
learn reverse CTMC rates or more general denoising Markov dynamics
\citep{campbell2022ctmc,sun2023score,benton2024denoising}. Discrete score
learning includes concrete and target-concrete score matching
\citep{meng2022concrete,zhang2025target}, while SEDD learns marginal
probability ratios with score entropy \citep{lou2024discrete}. Recent scalable
systems use score, clean-token denoiser, interpolating-kernel, or factorized-rate
parameterizations
\citep{sahoo2024simple,shi2024simplified,ou2025absorbing,
vonrutte2025gidd,li2026neuralctmc}. M2S retains SEDD's score-entropy CTMC and
changes how the complete score vector is parameterized.

\paragraph{Posterior and score parameterizations.}
Posterior parameterizations construct reverse dynamics from clean-token
predictions, as in D3PM and subsequent reparameterized or absorbing models
\citep{austin2021structured,he2023diffusionbert,zheng2024reparameterized}.
For absorbing corruption, RADD factors the concrete score through conditional
clean-data probabilities \citep{ou2025absorbing}, while MDLM and MD4 reduce
their objectives to weighted clean-token cross-entropy and MD4 highlights
score--mean consistency
\citep{sahoo2024simple,shi2024simplified}. GIDD interpolates between masked and
uniform corruption \citep{vonrutte2025gidd}, and concurrent work derives exact
coordinate-level score--denoiser conversions for uniform diffusion
\citep{gourevitch2026uniform}. These works characterize individual score
coordinates or population objectives. M2S instead asks whether the complete
SEDD score vector is induced by one clean-token posterior, audits this joint
condition in trained checkpoints, and enforces it by construction.

\section{Preliminaries}
\label{sec:preliminaries}

Let $\mathcal V=\{1,\ldots,K\}$ be a finite state space. A time-inhomogeneous
CTMC on $\mathcal V$ over $[0,T]$ is specified by a rate matrix $Q_t$ satisfying
$Q_t(a,b)\ge0$ for $a\neq b$ and $\sum_bQ_t(a,b)=0$.

\begin{definition}
\label{def:forward_transition_rate}
For some start time $s$ and end time $t=s+\Delta$ ($t>s$) as
$\Delta\to0$, we have
\begin{equation}
\label{eq:infinitesimal_transition}
q_{t\mid s}(b\mid a)
=\delta_{a,b}+Q_t(a,b)\Delta+o(\Delta),
\end{equation}
where $Q_t$ is called the forward transition rate.
\end{definition}

The Kolmogorov forward equation
$\frac{\mathrm d}{\mathrm dt}q_t=q_tQ_t$ determines the finite-time transition
kernel $P_{t\mid s}$~\citep{campbell2022ctmc}. For the corruption processes
considered here, we use the kernel in Definition~\ref{def:corruption_kernel},
which covers both uniform and absorbing corruption.

\begin{definition}
\label{def:corruption_kernel}
The cumulative transition probabilities of the CTMC are given by
\begin{equation}
\label{eq:forward_kernel}
q_{t\mid0}(a\mid z)
=\operatorname{Cat}\!\left(a;P_t(z,\cdot)\right),
\qquad
P_t=\alpha_tI+\beta_t\mathbf1\pi^\top,
\end{equation}
where $\alpha_t+\beta_t=1$, with $\alpha_0=1$ and $\alpha_T=0$, and $\pi$ is a
fixed distribution on $\mathcal V$. For uniform corruption,
$\pi=\mathbf1/K$; for absorbing corruption, $\pi=e_m$, where $e_m$ is the
one-hot vector for the mask state $m$. Here $\mathbf1$ is the all-ones vector.
As $t\to T$, $q_{t\mid0}(\cdot\mid z)\to\pi$ for every $z\in\mathcal V$, so
the reference distribution is $p_{\mathrm{ref}}=\pi$.
\end{definition}

For sequences, the forward process acts independently across sites. With
site-wise kernels $P_t^{(i)}$ and rates $Q_t^{(i)}$,
$q_t(x_t\mid x_0)=\prod_{i=1}^L P_t^{(i)}(x_t^i\mid x_0^i)$ and
$p_t(x)=\sum_{x_0}\pdata(x_0)\prod_{j=1}^L P_t^{(j)}(x^j\mid x_0^j)$.

\begin{definition}
\label{def:concrete_score}
For a sequence state $x\in\mathcal V^L$ and a candidate token $y\neq x^i$, let
$x^{i\to y}$ denote the sequence obtained by replacing site $i$ of $x$ with
$y$. Whenever $p_t(x)>0$, the concrete score and its corresponding reverse
transition rate are
\begin{equation}
\label{eq:concrete_score}
  s_i^\star(x,t;y)=\frac{p_t(x^{i\to y})}{p_t(x)},
  \qquad
  \overline Q_t^{(i)}(x^i,y\mid x)
  =Q_t^{(i)}(y,x^i)\,s_i^\star(x,t;y),
\end{equation}
respectively~\citep{campbell2022ctmc,lou2024discrete}.
\end{definition}

MDLM and GIDD use an $x_0$ parameterization, whereas SEDD directly predicts
the concrete score in Definition~\ref{def:concrete_score}. M2S instead predicts
the site-wise clean posterior
$\mu^i(x_t,t)\in\simplex$, with
\[
[\mu^i(x_t,t)]_z=\Prb(X_0^i=z\mid X_t=x_t),
\]
and maps it to the concrete score in Section~\ref{sec:methodology}.

\section{Methodology}
\label{sec:methodology}

This section presents the M2S bridge, its training loss,
and its connection to MD4. All proofs are provided in
Appendix~\ref{app:proofs}.

\begin{theorem}
\label{thm:bridge}
Under Assumption~\ref{ass:support}, for any site $i$, candidate token $y$,
and noisy state $x$ with $p_t(x)>0$, let
$\pi^\star(z)=\Prb(X_0^i=z\mid X_t=x)$. For every
$z\in\operatorname{supp}(\pi^\star)$, we have
$P_t^{(i)}(x^i\mid z)>0$. Then
\begin{equation}
\label{eq:bridge}
  s_i^\star(x,t;y)
  =
  \E\!\left[
    \frac{P_t^{(i)}(y\mid X_0^i)}
         {P_t^{(i)}(x^i\mid X_0^i)}
    \,\middle|\,X_t=x
  \right]
  =
  \sum_{z\in\operatorname{supp}(\pi^\star)}
  \pi^\star(z)\,
  \frac{P_t^{(i)}(y\mid z)}
       {P_t^{(i)}(x^i\mid z)} .
\end{equation}
\end{theorem}

Replacing the exact posterior in Theorem~\ref{thm:bridge} with the neural
network prediction $\mu_\theta^i(x,t)$ defines the M2S score as
\begin{equation}
\label{eq:m2s_score}
  s_{\theta,i}(x,t;y)
  =
  \sum_{z=1}^K[\mu_\theta^i(x,t)]_z\,
  \frac{P_t^{(i)}(y\mid z)}
       {P_t^{(i)}(x^i\mid z)}
  =: (B\mu_\theta^i)_y.
\end{equation}
For the uniform process, define the off-diagonal-to-diagonal kernel ratio as
$\rho_t=(\beta_t/K)/(\alpha_t+\beta_t/K)\in(0,1)$. At a noisy input $x_t$,
Equation~\eqref{eq:m2s_score} then simplifies, for each candidate
$y\neq x_t^i$, to
\begin{equation}
\label{eq:uniform_bridge}
  s_{\theta,i}(x_t,t;y)
  =1+(\rho_t-1)\big[\mu_\theta^i(x_t,t)\big]_{x_t^i}
    +(\rho_t^{-1}-1)\big[\mu_\theta^i(x_t,t)\big]_y.
\end{equation}
All candidate scores in Equation~\eqref{eq:uniform_bridge} are computed in
$O(K)$ time. We next introduce the training objective.

\begin{theorem}
\label{thm:elbo}
For the uniform CTMC with $\alpha_t=e^{-\sigma(t)}$ and
$\dot\sigma(t)\geq 0$, let
$x_0\sim\pdata$, $t\sim\mathrm{Unif}[0,1]$, and
$x_t\sim q_t(\cdot\mid x_0)$. For $y\neq x_t^i$, define
$r_i(x_0,x_t,t;y)=P_t^{(i)}(y\mid x_0^i)/P_t^{(i)}(x_t^i\mid x_0^i)$.
For $s,r>0$, define
$h(s,r)=s-r\log s+r\log r-r
=r\bigl(s/r-1-\log(s/r)\bigr)\ge0$, with equality if and only if $s=r$,
and extend this definition continuously to $r=0$ by $h(s,0)=s$. Define the
M2S objective
\begin{equation}
\label{eq:m2s_objective}
\begin{aligned}
  \mathcal L_{\MtwoS}(\theta)
  &=
  \E_{t,x_0,x_t}\!\left[
  \sum_{i=1}^L\sum_{y\neq x_t^i}w_{t,i}(x_t^i,y)\,
  h\bigl((B\mu_\theta^i)_y,r_i(x_0,x_t,t;y)\bigr)
  \right], \\
  w_{t,i}(x_t^i,y)
  &=Q_t^{(i)}(y,x_t^i)
  =\frac{\dot\sigma(t)}{K}
  \;\bigl(y\neq x_t^i\bigr).
\end{aligned}
\end{equation}
Let $p_{\theta,0}$ be the time-zero marginal obtained by initializing the
reverse process from $p_{\mathrm{ref}}$ and replacing $s_i^\star$ with
$s_{\theta,i}$ in its rates. Then
\begin{equation}
\label{eq:elbo_bound}
  \E_{x_0\sim\pdata}[-\log p_{\theta,0}(x_0)]
  \le
  \mathcal L_{\MtwoS}(\theta)
  +\E_{x_0\sim\pdata}
  D_{\mathrm{KL}}\!\left(q_T(\cdot\mid x_0)\,\middle\|\,p_{\mathrm{ref}}\right).
\end{equation}
\end{theorem}

The following results show that minimizing this objective recovers the true
score and the clean posterior.

\begin{proposition}
\label{prop:consistency}
For fixed $t,x,i,y$ with $s_i^\star(x,t;y)>0$, the conditional risk
$R_x(s):=\E[h(s,r_i)\mid X_t=x]$ satisfies
$R_x(s)-R_x(s_i^\star)=h(s,s_i^\star)\ge0$, with equality if and only if
$s=s_i^\star$.
\end{proposition}

\begin{theorem}
\label{thm:mu_optimal}
For fixed $t,x,i$, let $\pi^\star$ be the clean posterior from
Theorem~\ref{thm:bridge}, let
$\mathcal Y_+=\{y\neq x^i:w_{t,i}(x^i,y)>0\}$, and let $B_+$ contain the
rows of $B$ indexed by $\mathcal Y_+$. Assume $s_i^\star(x,t;y)>0$ for every
$y\in\mathcal Y_+$. Then $\mu$ minimizes the conditional M2S risk if and only
if $B_+\mu=B_+\pi^\star$. Moreover, $B_+$ is injective on $\simplex$ if and
only if
\begin{equation}
\label{eq:bridge_rank_condition}
  \ker(B_+)\cap\{v\in\R^K:\1^\top v=0\}=\{0\},
  \qquad\text{equivalently}\qquad
  \operatorname{rank}\!\begin{bmatrix}B_+\\ \1^\top\end{bmatrix}=K.
\end{equation}
When this rank condition holds, the unique minimizer is $\mu=\pi^\star$.
For the uniform kernel, $\dot\sigma(t)>0$ makes every candidate positively
weighted, and $\alpha_t\in(0,1)$ makes the augmented matrix in
Eq.~\eqref{eq:bridge_rank_condition} full rank. Therefore,
$\mu^\star=\pi^\star=\E[e_{X_0^i}\mid X_t=x]$.
\end{theorem}

Beyond the realizable case, this exact-recovery result suggests a canonical
repair for an arbitrary predicted score: project it onto the bridge image and
then invert the bridge. The uniform bridge is injective on the probability
simplex, so each realizable score vector corresponds to a unique posterior. Fix
$t,x,i$, let $k=x^i$, and
let $B_{t,k}\in\R^{(K-1)\times K}$ map a posterior to its scores for candidates
$y\neq k$. Given any score vector $s\in\R^{K-1}$, define the posterior whose
induced score is closest to $s$ by
\begin{equation}
\label{eq:projected_posterior}
  \mu^{\mathrm{proj}}(s)
  :=
  \operatorname*{arg\,min}_{\mu\in\simplex}
  \lVert B_{t,k}\mu-s\rVert_2^2.
\end{equation}

This projection distinguishes coordinate-wise feasibility from joint Bayes
realizability. Define the coordinate box and the bridge polytope by
\begin{equation}
\label{eq:score_realizability_sets}
  \mathcal C_{t,k}:=[\rho_t,\rho_t^{-1}]^{K-1},
  \qquad
  \mathcal P_{t,k}:=B_{t,k}(\simplex).
\end{equation}
Equation~\eqref{eq:uniform_bridge} gives
$\mathcal P_{t,k}\subseteq\mathcal C_{t,k}$, with strict inclusion for
$K>2$. Hence the score space has the following exact disjoint decomposition:
\begin{equation}
\label{eq:score_realizability_partition}
\R^{K-1}
=
\underbrace{\R^{K-1}\setminus\mathcal C_{t,k}}
            _{\text{coordinate-outside}}
\;\dot\cup\;
\underbrace{\mathcal C_{t,k}\setminus\mathcal P_{t,k}}
            _{\substack{\text{joint-only}\\\text{non-realizable}}}
\;\dot\cup\;
\underbrace{\mathcal P_{t,k}}
            _{\text{Bayes-realizable}}.
\end{equation}
The coordinate envelope is therefore necessary but not sufficient: a vector
can satisfy every scalar bound while its unique affine inverse has a negative
posterior component. Appendix~\ref{app:posthoc_projection} evaluates this same
decomposition on raw SEDD scores. To separate material violations from
sign-level floating-point effects, it divides the middle class at
$\varepsilon_\mu=10^{-6}$ into material and numerical-boundary subclasses;
their union corresponds to
$\mathcal C_{t,k}\setminus\mathcal P_{t,k}$ under the computed-sign convention.

\begin{proposition}
\label{prop:posterior_recovery}
Let $\pi^\star=\Prb(X_0^i=\cdot\mid X_t=x)$ and let
$s^\star=(s_i^\star(x,t;y))_{y\neq k}=B_{t,k}\pi^\star$. For the uniform
kernel with $\alpha_t\in(0,1)$, the projection in
Eq.~\eqref{eq:projected_posterior} uniquely recovers the clean posterior:
\begin{equation}
  \mu^{\mathrm{proj}}(s^\star)=\pi^\star.
\end{equation}
\end{proposition}

Proposition~\ref{prop:posterior_recovery} guarantees exact recovery only at the
population optimum. For an arbitrary learned score $s$, $\mu^{\mathrm{proj}}(s)$
is a valid posterior but need not be the true one. We next show that, for the
absorbing-mask process, the M2S objective in Eq.~\eqref{eq:m2s_objective}
reduces exactly to the MD4 loss~\citep{shi2024simplified}.

\begin{theorem}
\label{thm:md4_reduction}
Let $m$ be an absorbing mask. For every clean token $z\neq m$, let
$P_t(z\mid z)=\alpha_t$ and $P_t(m\mid z)=1-\alpha_t$, and let
$P_t(m\mid m)=1$, with no transitions between distinct clean tokens.
Suppose that each $\mu_\theta^i$ is a distribution over the clean vocabulary,
so $[\mu_\theta^i]_m=0$. Then
\begin{equation}
\label{eq:md4}
  \mathcal L_{\MtwoS}^{\mathrm{abs}}(\theta)
  =\mathcal L_{\mathrm{MD4}}(\theta)
  :=
  \E_{t,x_0,x_t}
  \Biggl[
  \sum_{i:\,x_t^i=m}
  \frac{-\dot\alpha_t}{1-\alpha_t}\,
  \bigl(-\log[\mu_\theta^i(x_t,t)]_{x_0^i}\bigr)
  \Biggr].
\end{equation}
\end{theorem}

\section{Experiments}
\label{sec:experiments}

\subsection{Same-Checkpoint Score Repair}
\label{sec:same_checkpoint_score_repair}

To isolate Bayes realizability at inference time from training differences, we generate
1,024 paired sequences from the epoch-32 pure-uniform SEDD checkpoint with the same released
sampler and coupled random-number stream. Replacing only
$s^{\mathrm{SEDD}}$ by its Euclidean projection onto the bridge polytope lowers external
GenPPL from $203.60$ to $175.07$; the paired average-NLL difference is
$-0.1510$ with 95\% bootstrap interval $[-0.1774,-0.1251]$. The complete
protocol, the distinction between direct, reconstructed, and projected scores,
and the partition relative to the bridge polytope are given in
Appendix~\ref{app:posthoc_projection}.

\subsection{Image and Language Generation}

We evaluate M2S on discrete image and language generation tasks. All
models use the uniform forward process
(Appendix~\ref{app:training_algorithm}) with $\alpha_t=1-t$ and
$\beta_t=t$. For a fair comparison, M2S adopts the same
DiT~\citep{peebles2023scalable} backbone and parameter count as the mainstream
baselines; full hyperparameters are given in
Appendix~\ref{app:exp_details}.

\textbf{Image Generation:}
\label{sec:exp_image}
We train M2S on MNIST, representing each $28\times28$ image as a sequence over
$\mathcal{S}=\{0,1,\ldots,255\}$. Using identical architectures, parameter counts,
optimization hyperparameters, training budgets, and samplers, M2S reduces FID
by more than 52 points relative to SEDD at every tested sampling budget, with
full details provided in Appendix~\ref{app:mnist_results}.

We also evaluate M2S on CIFAR-10 using a 28.4M-parameter U-Net with
self-attention. Following the MD4 setup, we match its total training budget of
512M image presentations while using pure-uniform corruption and the M2S
parameterization. Without data augmentation, the model reaches a test BPD
upper bound of $3.129$, compared with $3.173$ for SEDD under the same
evaluation protocol. Thus, M2S lowers test BPD by approximately $0.045$.
With the same 256-step Euler sampler and coupled random-number stream, M2S
also lowers FID-50k from $\CifarSEDDFID$ to $\CifarMtwoSFID$
(a $\CifarFIDReductionPct$ reduction); qualitative
samples and the complete protocol are reported in
Appendix~\ref{app:cifar10_results}.
It also improves over the D3PM and Campbell et al. results in
Table~\ref{tab:cifar10_results}. This result provides evidence that the
Bayes-realizable parameterization remains effective beyond MNIST. Full
training and evaluation details are given in
Appendix~\ref{app:cifar10_results}.

\begin{table}[H]
\centering
\small
\setlength{\tabcolsep}{5pt}
\caption{CIFAR-10 test BPD ($\downarrow$). M2S and SEDD are evaluated under
the same test protocol; the remaining baseline values are published results.
$\leq$ denotes a variational upper bound. The M2S and SEDD estimates are
rounded to three decimal places. M2S is trained without data augmentation.
A dash denotes an unreported parameter count.}
\label{tab:cifar10_results}
\begin{minipage}[t]{0.42\linewidth}
\centering
\vspace{0pt}
\setlength{\tabcolsep}{3.25pt}
\begin{tabular}{@{}lrS[table-format=1.2]@{}}
\toprule
Method & \# Params & {BPD ($\downarrow$)} \\
\midrule
\multicolumn{3}{@{}l}{\emph{Autoregressive}} \\
PixelRNN~\citep{oord2016pixel} & -- & 3.00 \\
Gated PixelCNN~\citep{oord2016conditional} & -- & 3.03 \\
PixelCNN++~\citep{salimans2017pixelcnn} & 53M & 2.92 \\
PixelSNAIL~\citep{chen2018pixelsnail} & 46M & 2.85 \\
Image Transformer~\citep{parmar2018image} & -- & 2.90 \\
Sparse Transformer~\citep{child2019generating} & 59M & 2.80 \\
\bottomrule
\end{tabular}
\end{minipage}\hfill
\begin{minipage}[t]{0.56\linewidth}
\centering
\vspace{0pt}
\setlength{\tabcolsep}{3.5pt}
\begin{tabular}{@{}lrr@{}}
\toprule
Method & \# Params & BPD ($\downarrow$) \\
\midrule
\multicolumn{3}{@{}l}{\emph{Pure-uniform discrete diffusion}} \\
\textbf{M2S (ours)} & \textbf{28.4M} & $\boldsymbol{\leq 3.129}$ \\
SEDD~\citep{lou2024discrete} & 28.4M & $\leq 3.173$ \\
\addlinespace[2pt]
\multicolumn{3}{@{}l}{\emph{Absorbing-mask discrete diffusion}} \\
D3PM Absorb~\citep{austin2021structured} & 37M & $\leq 4.40$ \\
Campbell et al. Absorb~\citep{campbell2022ctmc} & 28M & $\leq 3.52$ \\
MD4~\citep{shi2024simplified} & 28M & $\leq \mathbf{2.75}$ \\
\addlinespace[2pt]
\multicolumn{3}{@{}l}{\emph{Discrete-Gaussian diffusion}} \\
D3PM Gauss + logistic~\citep{austin2021structured} & 36M & $\leq 3.44$ \\
Campbell et al. ($\tau$LDR)~\citep{campbell2022ctmc} & 36M & $\leq 3.59$ \\
\bottomrule
\end{tabular}
\end{minipage}
\end{table}

\textbf{Language Modeling:}
\label{sec:exp_owt}
We compare M2S against SEDD~\citep{lou2024discrete},
MDLM~\citep{sahoo2024simple}, GIDD~\citep{vonrutte2025gidd}, and Neural
CTMC~\citep{li2026neuralctmc} on OpenWebText~\citep{Gokaslan2019OpenWeb}. To
ensure a fair comparison,
M2S and all baselines (SEDD, MDLM, GIDD, Neural CTMC) use an identical backbone
architecture (12-layer DiT-style Transformer, 768 hidden dim, 12 heads,
$\sim$169M parameters); the only difference across methods is
the parameterization of the reverse process and its associated loss. For GIDD
we report results with $p_{\mathrm{unif}}\in\{0.0,0.1,0.2\}$, where
$p_{\mathrm{unif}}=0.0$ corresponds to a pure mask process. For evaluation,
we draw 1024 unconditional samples from each model and score them with a
pretrained \textbf{Gemma2-9B} model to obtain \emph{generative perplexity}
(PPL); each method is run with multiple sampling seeds and we report the best
PPL across seeds.

\begingroup
\setlength{\intextsep}{2pt}
\begin{figure}[H]
\centering
\begin{subfigure}[t]{0.49\textwidth}
    \centering
    \includegraphics[width=\linewidth]{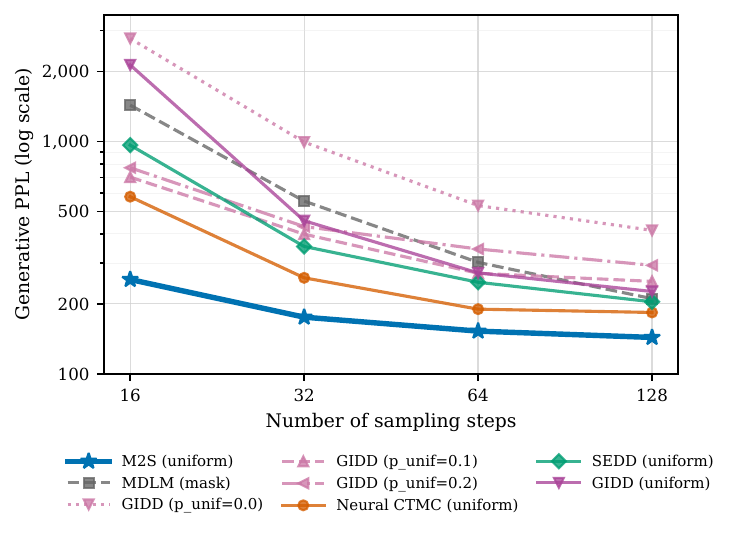}
    \caption{}
    \label{fig:owt_cross_method}
\end{subfigure}
\hfill
\begin{subfigure}[t]{0.49\textwidth}
    \centering
    \includegraphics[width=\linewidth]{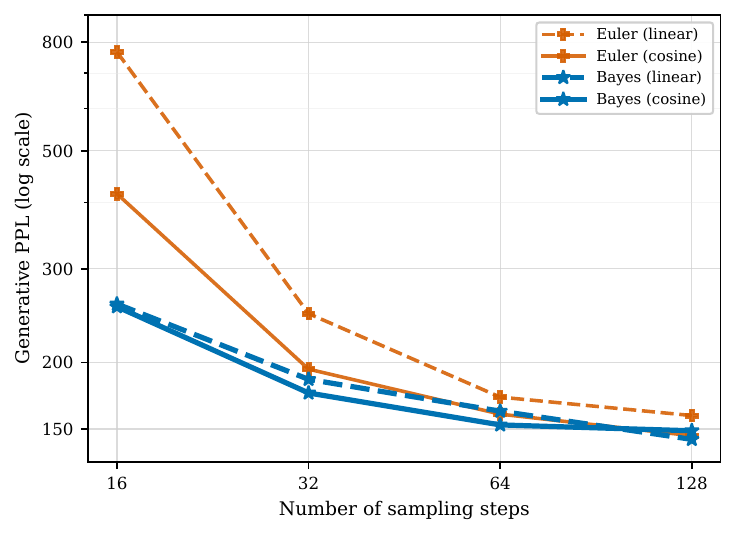}
    \caption{}
    \label{fig:owt_m2s_ablation}
\end{subfigure}
\caption{(a)~Generative-PPL comparison between M2S and the evaluated baselines.
(b)~Generative-PPL comparison of M2S Euler and Bayes samplers under linear and
cosine time grids.}
\label{fig:owt_generative_ppl}
\end{figure}
\endgroup

\vspace{-0.5em}
\begin{table}[H]
\centering
\caption{Generative perplexity ($\downarrow$) on OpenWebText for varying
numbers of sampling steps. All checkpoints use DiT-style backbones at a
comparable parameter scale. \textbf{Bold}: best overall per column;
\underline{underline}: best among rows with a reported 262B-token training
budget.}
\label{tab:owt_generative_ppl}
\small
\setlength{\tabcolsep}{3pt}
\begin{tabular}{@{}clccrrrr@{}}
\toprule
Type & Method & Train Toks & Max Len & 16 & 32 & 64 & 128 \\
\midrule
\multirow{3}{*}{mask}
  & SEDD~\citep{lou2024discrete} & 682B & 1024 & $825.5$ & $337.9$ & $186.5$ & $\mathbf{127.2}$ \\
  & MDLM~\citep{sahoo2024simple} & 262B & 1024 & $1432.8$ & $553.7$ & $301.6$ & $210.5$ \\
  & GIDD ($p_{\mathrm{unif}}=0.0$)~\citep{vonrutte2025gidd} & 262B & 512 & $2773.1$ & $993.7$ & $529.6$ & $414.3$ \\
\midrule
\multirow{2}{*}{mixture}
  & GIDD ($p_{\mathrm{unif}}=0.1$)~\citep{vonrutte2025gidd} & 262B & 512 & $702.0$ & $398.9$ & $270.8$ & $249.8$ \\
  & GIDD ($p_{\mathrm{unif}}=0.2$)~\citep{vonrutte2025gidd} & 262B & 512 & $770.4$ & $430.1$ & $344.3$ & $293.0$ \\
\midrule
\multirow{4}{*}{uniform}
  & Neural CTMC~\citep{li2026neuralctmc} & 262B & 512 & $578.3$ & $258.8$ & $189.7$ & $183.6$ \\
  & SEDD~\citep{lou2024discrete} & 262B & 512 & $963.9$ & $353.2$ & $247.7$ & $204.1$ \\
  & GIDD~\citep{vonrutte2025gidd} & 262B & 512 & $2134.0$ & $455.6$ & $271.7$ & $226.0$ \\
  & M2S (Bayes, cosine) & 262B & 512 & $\mathbf{\underline{254.6}}$ & $\mathbf{\underline{175.3}}$ & $\mathbf{\underline{152.6}}$ & $\underline{148.8}$ \\
\bottomrule
\end{tabular}
\par\vspace{2pt}{\footnotesize Note: (1) The losses of SEDD (mask) and GIDD
($p_{\mathrm{unif}}=0.0$) are equivalent to the MDLM loss. (2) Checkpoint and
sampler details are given in
Appendix~\ref{app:owt_experiments}.\par}
\end{table}
\vspace{-0.5em}

Table~\ref{tab:owt_generative_ppl} shows that M2S with Bayes sampling on the
cosine grid achieves the best overall PPL at 16--64 steps ($254.6$, $175.3$,
and $152.6$) and the best PPL among 262B-token checkpoints at 128 steps
($148.8$). Figure~\ref{fig:owt_generative_ppl}(b) compares the M2S sampling
configurations and shows that Bayes sampling on the linear grid achieves the
best 128-step PPL of $143.3$. Overall, M2S outperforms every evaluated uniform
and mixture model at all four
budgets. Against the strongest uniform baseline, M2S reduces PPL by $56\%$ at
16 steps ($578.3$ to $254.6$) and by $22\%$ at 128 steps ($183.6$ to $143.3$),
demonstrating a consistent advantage across the full sampling range. At 128
steps, M2S ranks second overall only to SEDD (mask), which uses 682B training
tokens compared with 262B for M2S.

\section{Conclusion}
\label{sec:conclusion}

This work identifies Bayes realizability as a structural requirement for
discrete score vectors. Positive SEDD scores define valid reverse CTMC rates,
but the complete vector need not be induced by any clean-token posterior. We
characterize this gap and introduce M2S, which predicts the one-site
clean-token posterior and maps it to all candidate scores through the known
forward kernel. The construction applies to known coordinate-wise kernels
under a mild support condition,
restricts uniform-corruption scores to the bridge polytope, and recovers MD4
under absorbing-mask corruption. Empirically, M2S improves FID over SEDD under
identical MNIST and CIFAR-10 settings and outperforms the evaluated pure-uniform baselines on
OpenWebText at every tested sampling budget. A fixed-checkpoint intervention
further shows that projecting SEDD scores onto the bridge polytope removes the
observed negative sampler weights and improves generative PPL. These results
establish Bayes realizability as a practical design principle for score-based
discrete diffusion.

\begingroup
\small
\bibliographystyle{plainnat}
\bibliography{references}
\endgroup

\newpage
\appendix
{\huge \bfseries Appendix}

\section*{Contents}
\startcontents[appendix]
\printcontents[appendix]{}{1}{\setcounter{tocdepth}{2}}

\vspace{1.5em}

\section{Proofs}
\label{app:proofs}

\begin{appthmframe}
\begin{assumption}[Support condition]
\label{ass:support}
Scores and posteriors are evaluated only at states $x$ with $p_t(x)>0$.
For every site $i$, candidate $y$, and clean token $z$ in the site-$i$ data
support,
\begin{equation*}
  P_t^{(i)}(x^i\mid z)=0
  \;\Longrightarrow\;
  P_t^{(i)}(y\mid z)=0 .
\end{equation*}
\end{assumption}
\end{appthmframe}

\begin{appthmbox}{Theorem~\ref{thm:bridge}}
Under Assumption~\ref{ass:support}, for any site $i$, candidate token $y$,
and noisy state $x$ with $p_t(x)>0$, let
$\pi^\star(z)=\Prb(X_0^i=z\mid X_t=x)$. For every
$z\in\operatorname{supp}(\pi^\star)$, we have
$P_t^{(i)}(x^i\mid z)>0$. Then
\begin{equation*}
  s_i^\star(x,t;y)
  =
  \E\!\left[
    \frac{P_t^{(i)}(y\mid X_0^i)}
         {P_t^{(i)}(x^i\mid X_0^i)}
    \,\middle|\,X_t=x
  \right]
  =
  \sum_{z\in\operatorname{supp}(\pi^\star)}
  \pi^\star(z)\,
  \frac{P_t^{(i)}(y\mid z)}
       {P_t^{(i)}(x^i\mid z)} .
\end{equation*}
\end{appthmbox}

\begin{proof}[Proof of Theorem~\ref{thm:bridge}]
For each clean token $z$, define
\begin{equation*}
  A_z
  :=\sum_{x_0:\,x_0^i=z}
  \pdata(x_0)\prod_{j\neq i}P_t^{(j)}(x^j\mid x_0^j),
  \qquad
  P_z:=P_t^{(i)}(x^i\mid z).
\end{equation*}
These quantities collect all contributions outside site $i$. They give
\begin{equation*}
  p_t(x)=\sum_z A_zP_z,
  \qquad
  p_t(x^{i\to y})=\sum_z A_zP_t^{(i)}(y\mid z).
\end{equation*}
The one-site posterior satisfies
\begin{equation*}
  \Prb(X_0^i=z\mid X_t=x)=\frac{A_zP_z}{p_t(x)}.
\end{equation*}
Let $\mathcal Z_\star=\operatorname{supp}(\pi^\star)
=\{z:A_zP_z>0\}$. For every $z\in\mathcal Z_\star$, we have $P_z>0$, so
define
\begin{equation*}
  g(z):=\frac{P_t^{(i)}(y\mid z)}{P_t^{(i)}(x^i\mid z)}
\end{equation*}
on $\mathcal Z_\star$. Expanding the conditional expectation over its support
gives
\begin{align*}
  \E[g(X_0^i)\mid X_t=x]
  &=\sum_{z\in\mathcal Z_\star}\pi^\star(z)g(z) \\
  &=\sum_{z\in\mathcal Z_\star}
    \frac{A_zP_z}{p_t(x)}
    \frac{P_t^{(i)}(y\mid z)}{P_z} \\
  &=\frac{1}{p_t(x)}
    \sum_{z\in\mathcal Z_\star}A_zP_t^{(i)}(y\mid z).
\end{align*}
If $z\notin\mathcal Z_\star$, then either $A_z=0$ or $P_z=0$. In the latter
case, whenever $A_z>0$, the token $z$ occurs in the clean support at site $i$,
so Assumption~\ref{ass:support} implies $P_t^{(i)}(y\mid z)=0$. Thus
$A_zP_t^{(i)}(y\mid z)=0$ outside $\mathcal Z_\star$, and the restricted sum
can be extended to all $z$. Hence
\begin{equation*}
  \E[g(X_0^i)\mid X_t=x]
  =\frac{\sum_zA_zP_t^{(i)}(y\mid z)}{p_t(x)}
  =\frac{p_t(x^{i\to y})}{p_t(x)}.
\end{equation*}
The first equality above is the posterior expansion in the theorem, and the
last equality is the concrete score $s_i^\star(x,t;y)$.
\end{proof}

\begin{appthmbox}{Theorem~\ref{thm:elbo}}
For the uniform CTMC with $\alpha_t=e^{-\sigma(t)}$ and
$\dot\sigma(t)\geq 0$, let
$x_0\sim\pdata$, $t\sim\mathrm{Unif}[0,1]$, and
$x_t\sim q_t(\cdot\mid x_0)$. For $y\neq x_t^i$, define
$r_i(x_0,x_t,t;y)=P_t^{(i)}(y\mid x_0^i)/P_t^{(i)}(x_t^i\mid x_0^i)$.
For $s,r>0$, define
$h(s,r)=s-r\log s+r\log r-r
=r\bigl(s/r-1-\log(s/r)\bigr)\ge0$, with equality if and only if $s=r$,
and extend this definition continuously to $r=0$ by $h(s,0)=s$. Define the
M2S objective
\begin{equation*}
\begin{aligned}
  \mathcal L_{\MtwoS}(\theta)
  &=
  \E_{t,x_0,x_t}\!\left[
  \sum_{i=1}^L\sum_{y\neq x_t^i}w_{t,i}(x_t^i,y)\,
  h\bigl((B\mu_\theta^i)_y,r_i(x_0,x_t,t;y)\bigr)
  \right], \\
  w_{t,i}(x_t^i,y)
  &=Q_t^{(i)}(y,x_t^i)
  =\frac{\dot\sigma(t)}{K}
  \;\bigl(y\neq x_t^i\bigr).
\end{aligned}
\end{equation*}
Let $p_{\theta,0}$ be the time-zero marginal obtained by initializing the
reverse process from $p_{\mathrm{ref}}$ and replacing $s_i^\star$ with
$s_{\theta,i}$ in its rates. Then
\begin{equation*}
  \E_{x_0\sim\pdata}[-\log p_{\theta,0}(x_0)]
  \le
  \mathcal L_{\MtwoS}(\theta)
  +\E_{x_0\sim\pdata}
  D_{\mathrm{KL}}\!\left(q_T(\cdot\mid x_0)\,\middle\|\,p_{\mathrm{ref}}\right).
\end{equation*}
\end{appthmbox}

\begin{proof}[Proof of Theorem~\ref{thm:elbo}]
For $u=s/r>0$, the inequality $u-1-\log u\ge0$ gives
$h(s,r)=r(u-1-\log u)\ge0$, with equality only at $u=1$.

Theorem~3.6 of \citet{lou2024discrete}, based on the continuous-time ELBO
of \citet{campbell2022ctmc}, includes the target-only normalization
$r(\log r-1)=r\log r-r$. Its integrand is therefore exactly $h(s,r)$.
Since $T=1$, the theorem gives, for every $x_0$,
\begin{align*}
  -\log p_{\theta,0}(x_0)
  \le{}&
  \int_0^1
  \E_{x_t\sim q_t(\cdot\mid x_0)}
  \Biggl[
    \sum_{\widetilde x\neq x_t}
    Q_t(\widetilde x,x_t)\,
    h\!\left(
      s_\theta(x_t,t;\widetilde x),
      \frac{q_t(\widetilde x\mid x_0)}{q_t(x_t\mid x_0)}
    \right)
  \Biggr] \mathrm dt \\
  &+D_{\mathrm{KL}}\!\left(q_T(\cdot\mid x_0)\,\middle\|\,p_{\mathrm{ref}}\right),
\end{align*}
written here in our row-generator convention.

The coordinate-wise generator has a nonzero off-diagonal entry only when
$\widetilde x=x_t^{i\to y}$ for some $i$ and $y\neq x_t^i$. For this pair,
\begin{equation*}
  Q_t(\widetilde x,x_t)=Q_t^{(i)}(y,x_t^i),
  \qquad
  \frac{q_t(\widetilde x\mid x_0)}{q_t(x_t\mid x_0)}
  =\frac{P_t^{(i)}(y\mid x_0^i)}
         {P_t^{(i)}(x_t^i\mid x_0^i)}
  =r_i(x_0,x_t,t;y).
\end{equation*}
For the uniform process,
$Q_t^{(i)}(y,x_t^i)=\dot\sigma(t)/K$ for $y\neq x_t^i$. Hence the integral,
after averaging over $x_0\sim\pdata$, is exactly
$\mathcal L_{\MtwoS}(\theta)$ because $t\sim\mathrm{Unif}[0,1]$. Averaging
the pointwise bound proves Eq.~\eqref{eq:elbo_bound}.
\end{proof}

\begin{appthmbox}{Proposition~\ref{prop:consistency}}
For fixed $t,x,i,y$ with $s_i^\star(x,t;y)>0$, the conditional risk
$R_x(s):=\E[h(s,r_i)\mid X_t=x]$ satisfies
$R_x(s)-R_x(s_i^\star)=h(s,s_i^\star)\ge0$, with equality if and only if
$s=s_i^\star$.
\end{appthmbox}

\begin{proof}[Proof of Proposition~\ref{prop:consistency}]
Condition on $X_t=x$ and write $s^\star=s_i^\star(x,t;y)$. Since
$\E[r_i\mid X_t=x]=s^\star$ by Theorem~\ref{thm:bridge},
\begin{align*}
  R_x(s)
  &:=\E[h(s,r_i)\mid X_t=x] \\
  &=s-s^\star\log s+C_x,
\end{align*}
where $C_x$ is independent of $s$. Moreover,
\begin{equation*}
  R_x'(s)=1-\frac{s^\star}{s},
  \qquad
  R_x''(s)=\frac{s^\star}{s^2}>0.
\end{equation*}
Thus $R_x$ is strictly convex and uniquely minimized at $s^\star$, with
\begin{equation*}
  R_x(s)-R_x(s^\star)
  =s-s^\star-s^\star\log\frac{s}{s^\star}
  =h(s,s^\star).
\end{equation*}
\end{proof}

\begin{appthmbox}{Theorem~\ref{thm:mu_optimal}}
For fixed $t,x,i$, let $\pi^\star$ be the clean posterior from
Theorem~\ref{thm:bridge}, let
$\mathcal Y_+=\{y\neq x^i:w_{t,i}(x^i,y)>0\}$, and let $B_+$ contain the
rows of $B$ indexed by $\mathcal Y_+$. Assume $s_i^\star(x,t;y)>0$ for every
$y\in\mathcal Y_+$. Then $\mu$ minimizes the conditional M2S risk if and only
if $B_+\mu=B_+\pi^\star$. Moreover, $B_+$ is injective on $\simplex$ if and
only if
\begin{equation*}
  \ker(B_+)\cap\{v\in\R^K:\1^\top v=0\}=\{0\},
  \qquad\text{equivalently}\qquad
  \operatorname{rank}\!\begin{bmatrix}B_+\\ \1^\top\end{bmatrix}=K.
\end{equation*}
When this rank condition holds, the unique minimizer is $\mu=\pi^\star$.
For the uniform kernel, $\dot\sigma(t)>0$ makes every candidate positively
weighted, and $\alpha_t\in(0,1)$ makes the augmented matrix in
Eq.~\eqref{eq:bridge_rank_condition} full rank,
and hence $\mu^\star=\pi^\star=\E[e_{X_0^i}\mid X_t=x]$.
\end{appthmbox}

\begin{proof}[Proof of Theorem~\ref{thm:mu_optimal}]
By Theorem~\ref{thm:bridge},
\begin{equation*}
  B\pi^\star=s_i^\star(x,t;\cdot).
\end{equation*}
Write $w_y=w_{t,i}(x^i,y)$ and let $\mathcal R_x(\mu)$ denote the conditional
M2S risk summed over $y\in\mathcal Y_+$. Applying
Proposition~\ref{prop:consistency} to each candidate gives, for every $\mu$
with finite conditional risk,
\begin{align*}
  \mathcal R_x(\mu)-\mathcal R_x(\pi^\star)
  &=\sum_{y\in\mathcal Y_+}w_y\,
    h\bigl((B\mu)_y,(B\pi^\star)_y\bigr)\\
  &\geq0.
\end{align*}
Here every $w_y$ and $(B\pi^\star)_y=s_i^\star(x,t;y)$ is positive. Since
$h(s,r)=0$ if and only if $s=r$, equality holds exactly when
\begin{equation*}
  B_+\mu=B_+\pi^\star.
\end{equation*}
This proves the first claim.

We next prove the rank characterization. If $\mu,\nu\in\simplex$ satisfy
$B_+\mu=B_+\nu$, then $v=\mu-\nu$ satisfies
\begin{equation*}
  B_+v=0,
  \qquad
  \1^\top v=\1^\top\mu-\1^\top\nu=0.
\end{equation*}
Thus the null-space condition in Eq.~\eqref{eq:bridge_rank_condition} implies
$v=0$ and hence $\mu=\nu$, so $B_+$ is injective on $\simplex$.

For the converse, suppose there exists a nonzero
$v\in\ker(B_+)$ with $\1^\top v=0$. Let
\begin{equation*}
  \bar\mu=\frac1K\1,
  \qquad
  0<\varepsilon<\frac{1}{K\lVert v\rVert_\infty},
  \qquad
  \mu^\pm=\bar\mu\pm\varepsilon v.
\end{equation*}
The choice of $\varepsilon$ ensures that every coordinate of $\mu^\pm$ is
positive, while $\1^\top v=0$ gives $\1^\top\mu^\pm=1$. Hence
$\mu^+,\mu^-\in\simplex$ and $\mu^+\neq\mu^-$. However,
\begin{equation*}
  B_+\mu^+-B_+\mu^-=2\varepsilon B_+v=0,
\end{equation*}
so $B_+$ is not injective on $\simplex$. This proves
\begin{equation*}
  B_+\text{ is injective on }\simplex
  \quad\Longleftrightarrow\quad
  \ker(B_+)\cap\ker(\1^\top)=\{0\}.
\end{equation*}
Finally,
\begin{equation*}
  \ker\!\begin{bmatrix}B_+\\ \1^\top\end{bmatrix}
  =\ker(B_+)\cap\ker(\1^\top).
\end{equation*}
Because the augmented matrix has $K$ columns, its kernel is trivial if and
only if its column rank is $K$. This establishes both equivalences in
Eq.~\eqref{eq:bridge_rank_condition}. When they hold,
$B_+\mu=B_+\pi^\star$ forces $\mu=\pi^\star$.

For the uniform kernel, $\dot\sigma(t)>0$ makes every candidate $y\neq x^i$
positively weighted, so $B_+=B$. Moreover, $\alpha_t\in(0,1)$ implies
$\rho_t\in(0,1)$. Let $k=x^i$ and suppose $\mu,\nu\in\simplex$ satisfy
$B\mu=B\nu$. For any $\eta\in\simplex$, summing
Eq.~\eqref{eq:uniform_bridge} over $y\neq k$ gives
\begin{align*}
  F(\eta)
  &:=\sum_{y\neq k}\bigl((B\eta)_y-1\bigr)\\
  &=(K-1)(\rho_t-1)\eta_k
    +(\rho_t^{-1}-1)(1-\eta_k).
\end{align*}
The coefficient of $\eta_k$ is
\begin{equation*}
  (K-1)(\rho_t-1)-(\rho_t^{-1}-1)
  =(\rho_t-1)\bigl((K-1)+\rho_t^{-1}\bigr),
\end{equation*}
which is nonzero for $\rho_t\in(0,1)$. Since $B\mu=B\nu$ implies
$F(\mu)=F(\nu)$, it follows that $\mu_k=\nu_k$. For each $y\neq k$,
Eq.~\eqref{eq:uniform_bridge} then gives
\begin{align*}
  0
  &=(B\mu)_y-(B\nu)_y\\
  &=(\rho_t-1)(\mu_k-\nu_k)
    +(\rho_t^{-1}-1)(\mu_y-\nu_y)\\
  &=(\rho_t^{-1}-1)(\mu_y-\nu_y).
\end{align*}
Because $\rho_t^{-1}-1\neq0$, we obtain $\mu_y=\nu_y$ for every $y\neq k$.
Thus $\mu=\nu$, proving that the uniform bridge is injective on $\simplex$;
by the equivalence above, its augmented matrix has rank $K$. The unique
minimizer is therefore
$\mu=\pi^\star=\E[e_{X_0^i}\mid X_t=x]$.
\end{proof}

\begin{appthmbox}{Proposition~\ref{prop:posterior_recovery}}
Let $\pi^\star=\Prb(X_0^i=\cdot\mid X_t=x)$ and let
$s^\star=(s_i^\star(x,t;y))_{y\neq k}=B_{t,k}\pi^\star$. For the uniform
kernel with $\alpha_t\in(0,1)$, the projection in
Eq.~\eqref{eq:projected_posterior} uniquely recovers the clean posterior:
\begin{equation*}
  \mu^{\mathrm{proj}}(s^\star)=\pi^\star.
\end{equation*}
\end{appthmbox}

\begin{proof}[Proof of Proposition~\ref{prop:posterior_recovery}]
Because $\pi^\star\in\simplex$ and $s^\star=B_{t,k}\pi^\star$, choosing
$\mu=\pi^\star$ in Eq.~\eqref{eq:projected_posterior} gives objective value
zero. Hence every minimizer $\widehat\mu$ satisfies
$B_{t,k}\widehat\mu=s^\star=B_{t,k}\pi^\star$. The uniform bridge is
injective on $\simplex$ when $\alpha_t\in(0,1)$ by
Theorem~\ref{thm:mu_optimal}, so $\widehat\mu=\pi^\star$. Thus the minimizer
is unique and equals the clean posterior.
\end{proof}

\begin{appthmbox}{Theorem~\ref{thm:md4_reduction}}
Let $m$ be an absorbing mask. For every clean token $z\neq m$, let
$P_t(z\mid z)=\alpha_t$ and $P_t(m\mid z)=1-\alpha_t$, and let
$P_t(m\mid m)=1$, with no transitions between distinct clean tokens.
Suppose that each $\mu_\theta^i$ is a distribution over the clean vocabulary,
so $[\mu_\theta^i]_m=0$. Then
\begin{equation*}
  \mathcal L_{\MtwoS}^{\mathrm{abs}}(\theta)
  =\mathcal L_{\mathrm{MD4}}(\theta)
  :=
  \E_{t,x_0,x_t}
  \Biggl[
  \sum_{i:\,x_t^i=m}
  \frac{-\dot\alpha_t}{1-\alpha_t}\,
  \bigl(-\log[\mu_\theta^i(x_t,t)]_{x_0^i}\bigr)
  \Biggr].
\end{equation*}
\end{appthmbox}

\begin{proof}[Proof of Theorem~\ref{thm:md4_reduction}]
\textbf{Masked sites.}
Fix a site with $x_t^i=m$ and let $c_t=\alpha_t/(1-\alpha_t)$. The bridge
and conditional target reduce to
\begin{equation*}
  s_{\theta,i}(x_t,t;y)=c_t[\mu_\theta^i(x_t,t)]_y,
  \qquad
  r_i(x_0,x_t,t;y)=c_t\ind\{y=x_0^i\}.
\end{equation*}
Using $h(s,0)=s$, substitution into the inner score-entropy sum in
Eq.~\eqref{eq:m2s_objective} gives
\begin{equation*}
\sum_{y\neq m}\Bigl[
  c_t[\mu_\theta^i]_y
  -c_t\ind\{y=x_0^i\}\log(c_t[\mu_\theta^i]_y)
  +c_t\ind\{y=x_0^i\}\log c_t
  -c_t\ind\{y=x_0^i\}
\Bigr].
\end{equation*}
Because $\mu_\theta^i$ is supported on the clean vocabulary,
$\sum_{y\neq m}[\mu_\theta^i]_y=1$. The first and last terms therefore
cancel after summation, and the two occurrences of $\log c_t$ also cancel.
The remaining loss is
\begin{equation*}
  -c_t\log[\mu_\theta^i]_{x_0^i}.
\end{equation*}
The Kolmogorov equation gives the clean-to-mask rate
$Q_t^{(i)}(y,m)=-\dot\alpha_t/\alpha_t$, independent of the clean candidate
$y$. Therefore,
\begin{equation*}
  Q_t^{(i)}(y,m)\,c_t
  =\frac{-\dot\alpha_t}{\alpha_t}
    \frac{\alpha_t}{1-\alpha_t}
  =\frac{-\dot\alpha_t}{1-\alpha_t}.
\end{equation*}
Thus, summing over masked sites and taking the expectation gives exactly
Eq.~\eqref{eq:md4}.

\textbf{Unmasked sites.}
Now let $x_t^i=x_0^i=k\neq m$. For every candidate $y\neq k$,
$Q_t^{(i)}(y,k)=0$: clean-to-clean rates vanish, and the absorbing state has
no outgoing rate. Thus every candidate weight is zero, so unmasked sites
make no contribution. This completes the reduction.
\end{proof}

\section{Algorithms}
\label{app:training_algorithm}

\subsection{Training Algorithm}
\label{app:m2s_training_algorithm}

During training, a clean token $z$ is corrupted according to the uniform
kernel
\begin{equation}
  q_t(y\mid z)
  =\alpha_t\ind\{y=z\}+\frac{1-\alpha_t}{K},
  \qquad
  \alpha_t=1-t,
  \label{eq:training_uniform_forward}
\end{equation}
where $y$ is the corrupted token, $K$ is the vocabulary size, and
$t\in[0,1]$.

Algorithm~\ref{alg:m2s_training} summarizes the training procedure. Given a
clean sequence $x_0$ and a diffusion time $t$, we sample $x_t$ from the forward
kernel in Eq.~\eqref{eq:training_uniform_forward}. The model predicts the
clean-token posterior, converts it to concrete scores through the M2S bridge,
and is optimized with the loss in Eq.~\eqref{eq:m2s_objective}.

\begin{algorithm}[H]
\caption{M2S Training}
\label{alg:m2s_training}
\begin{algorithmic}[1]
\Require model $f_\theta$, data distribution $p_{\mathrm{data}}$, vocabulary
size $K$, learning rate $\eta$
\While{not converged}
  \State Sample $x_0\sim p_{\mathrm{data}}$ and $t\sim\mathrm{Unif}[0,1]$
  \State Corrupt $x_0$ with Eq.~\eqref{eq:training_uniform_forward} to obtain
  $x_t$
  \State Predict $\mu_\theta^i=\softmax(f_\theta^i(x_t,t))$ for all sites $i$
  \State Compute $s_{\theta,i}(x_t,t;y)$ for all $i$ and $y\neq x_t^i$ using
  Eq.~\eqref{eq:m2s_score}
  \State Compute $r_i$ and $w_{t,i}$ from the known forward process as in
  Eq.~\eqref{eq:m2s_objective}
  \State $\theta\gets\theta-\eta\nabla_\theta\mathcal L_{\MtwoS}(\theta)$
\EndWhile
\State \Return $\theta$
\end{algorithmic}
\end{algorithm}

\subsection{Sampling Algorithms}
\label{app:m2s_sampling_algorithms}

M2S pairs Euler and Bayes updates with either a linear or cosine time grid.
This gives four samplers: Euler-linear, Euler-cosine, Bayes-linear, and
Bayes-cosine.

For $M$ sampling steps, let $t_0>t_1>\cdots>t_M$ denote the reverse-time grid:
\begin{equation}
  t_m^{\mathrm{lin}}=1-\frac{m}{M},
  \qquad
  t_m^{\mathrm{cos}}
  =\cos\!\left(\frac{\pi m}{2M}\right),
  \qquad m=0,\ldots,M,
  \label{eq:sampling_time_grids}
\end{equation}
Let $\Delta_m=t_m-t_{m+1}$. Both grids start at $t_0=1$ and end at $t_M=0$.
All samplers initialize $X_{t_0}$ from the uniform distribution and update all
sites in parallel. At step $m$, the model predicts the clean-token posterior
$\mu_{\theta,i}$ from the current sequence $x=X_{t_m}$ and maps it to scores
using Eq.~\eqref{eq:uniform_bridge}. Euler sampling uses the reverse rates
\begin{equation}
  R_{m,i}(y)
  =Q_{t_m}^{(i)}(y,x^i)\,s_{\theta,i}(x,t_m;y),
  \qquad y\neq x^i.
  \label{eq:sampling_reverse_rate}
\end{equation}

\begin{algorithm}[H]
\caption{M2S Euler Sampling (Linear or Cosine Grid)}
\label{alg:m2s_euler_sampling}
\begin{algorithmic}[1]
\Require model $f_\theta$, steps $M$, grid type
$g\in\{\mathrm{lin},\mathrm{cos}\}$, vocabulary size $K$, sequence length $L$
\State Construct $\{t_m\}_{m=0}^M$ from Eq.~\eqref{eq:sampling_time_grids}
\State Sample $X_{t_0}^i\sim\mathrm{Unif}(\{1,\ldots,K\})$ for all sites $i$
\For{$m=0,\ldots,M-1$}
  \State $x\gets X_{t_m}$ and $\Delta_m\gets t_m-t_{m+1}$
  \State Predict $\mu_{\theta,i}$ and compute scores using
  Eq.~\eqref{eq:uniform_bridge}
  \State Compute $R_{m,i}$ using Eq.~\eqref{eq:sampling_reverse_rate}
  \ForAll{sites $i$ \textbf{in parallel}}
    \State $p_i(y)\gets\Delta_mR_{m,i}(y)$ for $y\neq x^i$
    \State $p_i(x^i)\gets1-\sum_{y\neq x^i}p_i(y)$
    \State Sample $X_{t_{m+1}}^i\sim\mathrm{Categorical}(p_i)$
  \EndFor
\EndFor
\State \Return $X_{t_M}$
\end{algorithmic}
\end{algorithm}

Bayes sampling draws directly from the finite-time posterior instead of
discretizing the reverse rates. For the current token $k=X_{t_m}^i$,
marginalizing over $\mu_{\theta,i}$ gives
\begin{equation}
  \widetilde p_{m,i}(y)
  =\sum_{z=1}^K
   [\mu_{\theta,i}]_z
   \frac{q_{t_{m+1}}(y\mid z)\,
         q_{t_m\mid t_{m+1}}(k\mid y)}
        {q_{t_m}(k\mid z)}.
  \label{eq:bayes_posterior_update}
\end{equation}
\begin{algorithm}[H]
\caption{M2S Bayes Sampling (Linear or Cosine Grid)}
\label{alg:m2s_bayes_sampling}
\begin{algorithmic}[1]
\Require model $f_\theta$, steps $M$, grid type
$g\in\{\mathrm{lin},\mathrm{cos}\}$, vocabulary size $K$, sequence length $L$
\State Construct $\{t_m\}_{m=0}^M$ from Eq.~\eqref{eq:sampling_time_grids}
\State Sample $X_{t_0}^i\sim\mathrm{Unif}(\{1,\ldots,K\})$ for all sites $i$
\For{$m=0,\ldots,M-1$}
  \State $x\gets X_{t_m}$
  \State Predict $\mu_{\theta,i}$ from $(x,t_m)$
  \ForAll{sites $i$ \textbf{in parallel}}
    \State Compute $\widetilde p_{m,i}(y)$ for all $y$ using
    Eq.~\eqref{eq:bayes_posterior_update}
    \State Normalize $p_{m,i}\gets\widetilde p_{m,i}/
    \sum_y\widetilde p_{m,i}(y)$
    \State Sample $X_{t_{m+1}}^i\sim\mathrm{Categorical}(p_{m,i})$
  \EndFor
\EndFor
\State \Return $X_0$
\end{algorithmic}
\end{algorithm}

\section{Experimental Details}
\label{app:exp_details}

\subsection{MNIST Experiments}
\label{app:mnist_results}
\label{sec:exp_mnist}

We compare M2S and SEDD on $28\times28$ MNIST images quantized to $K=256$
gray levels. For a fair comparison, both methods use the same width-64 U-Net
architecture and parameter count, optimization hyperparameters (AdamW with a
learning rate of $2\times10^{-4}$), uniform transition kernel and forward
process $\alpha_t=e^{-t}$, 50-epoch training budget and generate samples with
the same sampler, while differing only in the training parameterization: M2S
predicts the posterior mean, whereas SEDD predicts the log-score directly. We
generate 5000 samples for each model and compute FID against the first 5000
MNIST test images.

Table~\ref{tab:mnist_sweep} compares the two methods with 50, 100, and 200
sampling steps while holding all other settings fixed. M2S lowers FID by more
than 52 points at every budget.

\begin{table}[H]
\centering
\caption{MNIST FID ($\downarrow$) across sampling budgets.}
\label{tab:mnist_sweep}
\begin{tabular}{lccc}
\toprule
Objective & 50 steps & 100 steps & 200 steps \\
\midrule
M2S (ours) & $\mathbf{72.80}$ & $\mathbf{72.87}$ & $\mathbf{73.44}$ \\
SEDD & $127.18$ & $126.81$ & $126.19$ \\
\bottomrule
\end{tabular}
\end{table}

Figure~\ref{fig:mnist_samples} shows the first 64 uncurated samples from the
200-step runs. M2S produces clearer strokes and fewer isolated bright pixels
than SEDD, consistent with the FID comparison.

\begin{figure}[H]
\centering
\begin{subfigure}[t]{0.31\textwidth}
    \centering
    \includegraphics[width=\linewidth]{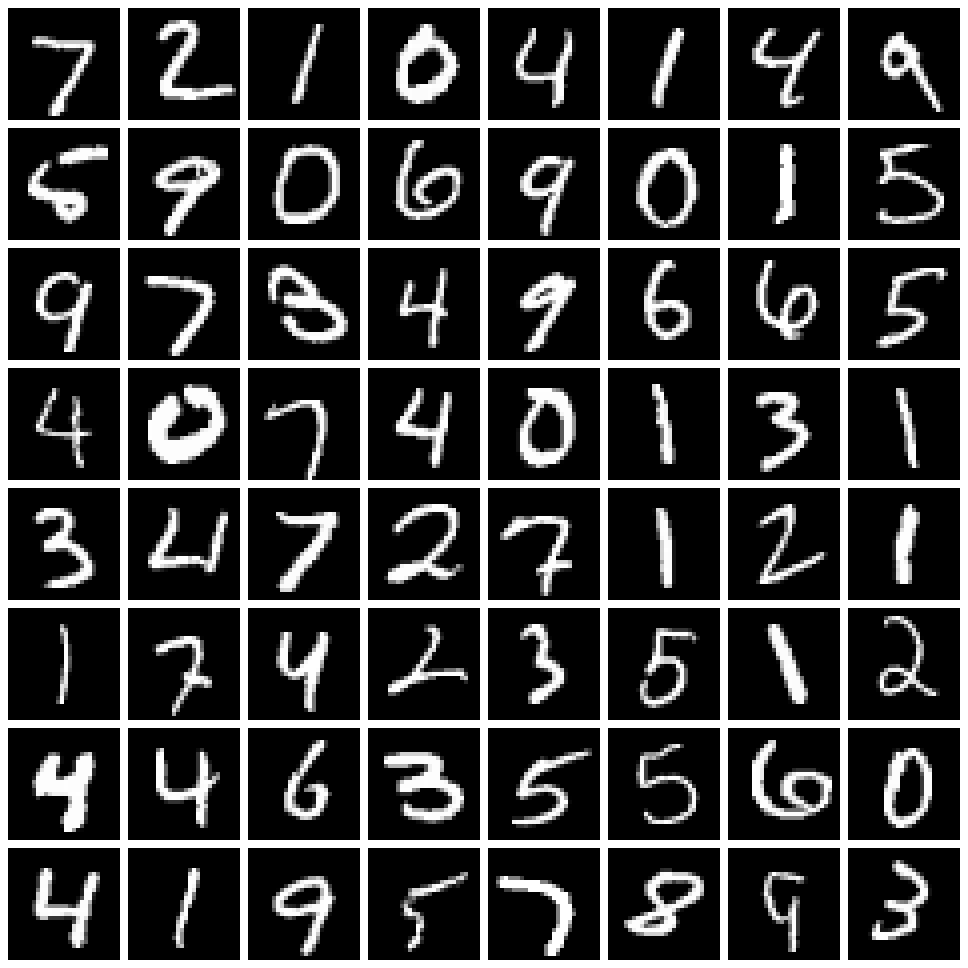}
    \caption{MNIST test data}
\end{subfigure}
\hfill
\begin{subfigure}[t]{0.31\textwidth}
    \centering
    \includegraphics[width=\linewidth]{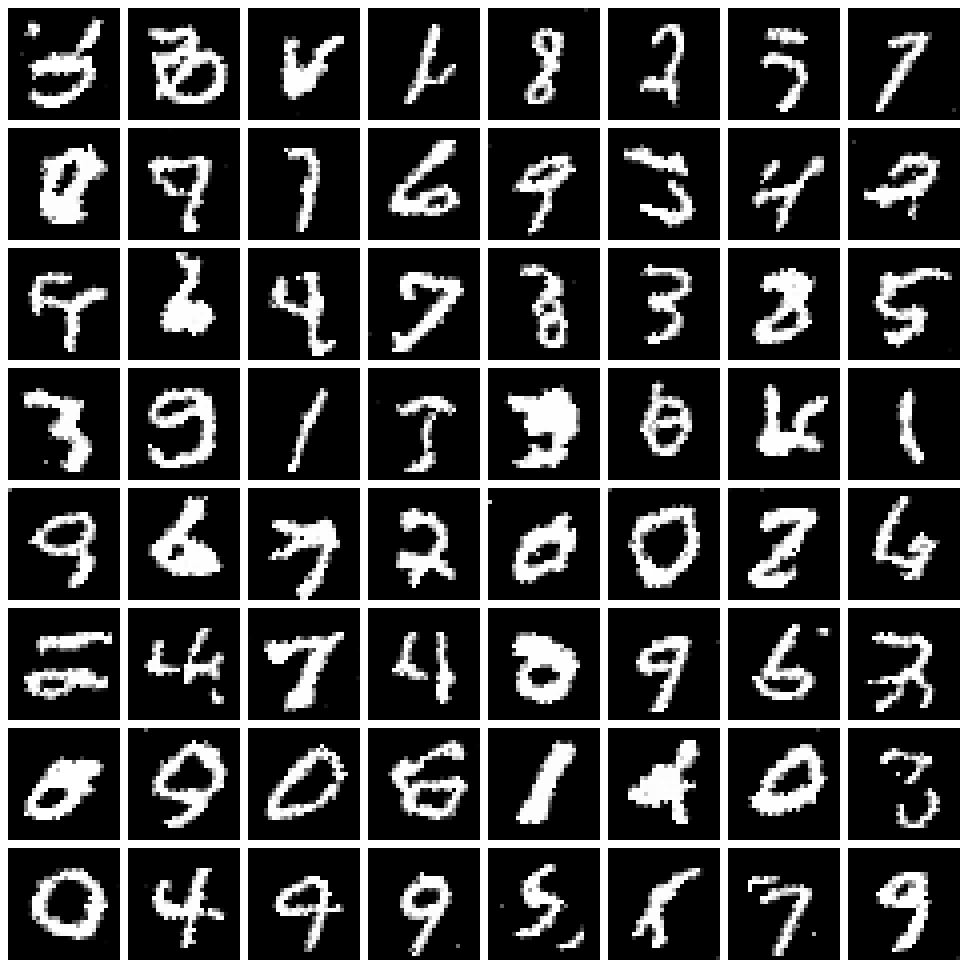}
    \caption{M2S, FID $73.4$}
\end{subfigure}
\hfill
\begin{subfigure}[t]{0.31\textwidth}
    \centering
    \includegraphics[width=\linewidth]{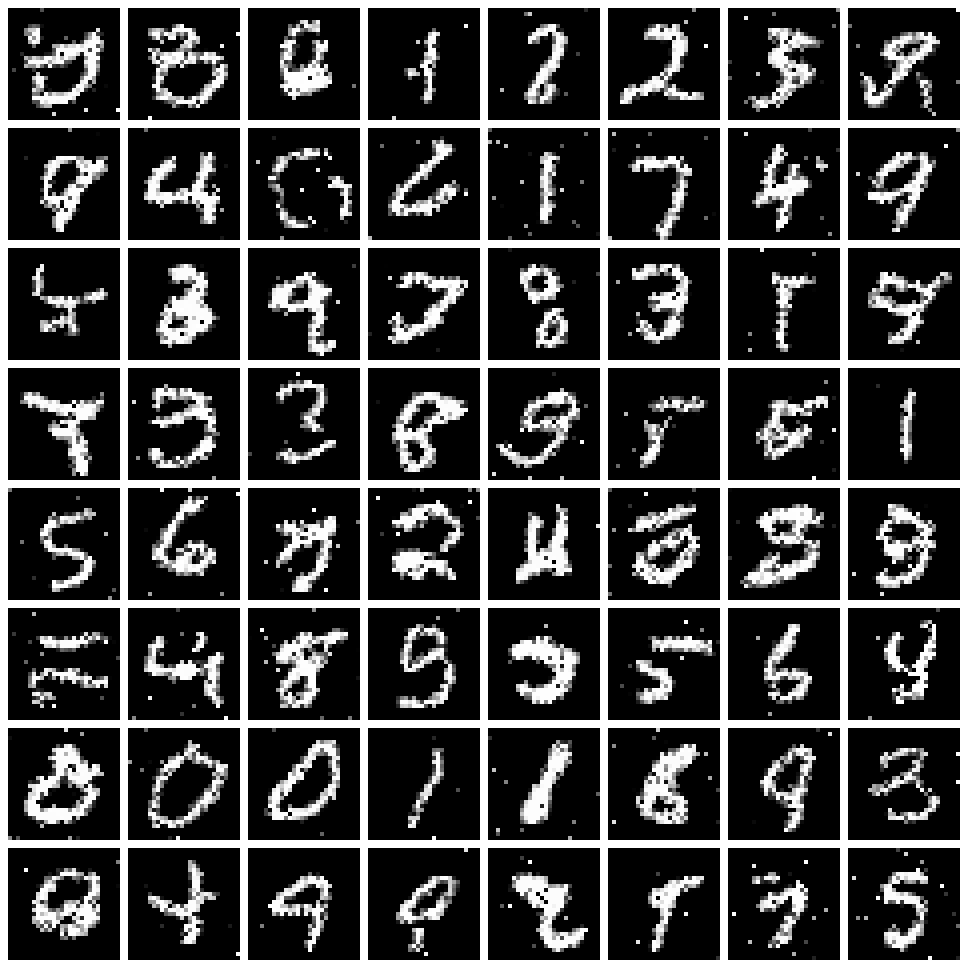}
    \caption{SEDD, FID $126.2$}
\end{subfigure}
\caption{Qualitative comparison on MNIST.}
\label{fig:mnist_samples}
\end{figure}

\subsection{CIFAR-10 Experiments}
\label{app:cifar10_results}

\paragraph{Training setup.}
We follow the CIFAR-10~\citep{krizhevsky2009learning} setup of
MD4~\citep{shi2024simplified}. Each
$32\times32$ RGB image is represented as a length-3072 sequence with 256
states per color channel. We use the same approximately 28M-parameter U-Net
with self-attention, AdamW with learning rate $4\times10^{-4}$ and weight
decay $0.01$. The warmup covers 25,600 image presentations, equal to 100
MD4 batch-256 updates and 6.25 of our large-batch updates, and is followed by
cosine learning-rate decay.
The model is trained without data augmentation. M2S retains the clean-token
posterior head but uses pure-uniform corruption and maps its output to concrete
scores as described in Section~\ref{sec:methodology}.

MD4 trains for 2M updates with batch size 256. We use a global batch size of
4096 for 125,000 updates, matching its total budget of 512M image
presentations (10,240 CIFAR-10-equivalent epochs). The resulting
28,427,520-parameter checkpoint achieves a CIFAR-10
test BPD upper bound of $3.129$. Re-evaluating SEDD under exactly the same
test protocol gives $3.173$, so M2S lowers test BPD by approximately $0.045$.

\paragraph{Optimization trace.}
The two parameterizations are trained with the same complete score-entropy
objective: M2S maps posterior logits to a log-score through the uniform bridge,
whereas SEDD predicts the log-score directly, after which both use the
same time distribution, corruption kernel, token sum, and batch mean. We log
the all-reduced raw objective over all 16 ranks at every optimizer step and
convert step $u$ to $4096u/50{,}000$ CIFAR-10-equivalent epochs. M2S starts
higher and optimizes more slowly early in training, but its trailing 625-step
mean makes its final crossover below SEDD at equivalent epoch
$3963.6$ and remains lower thereafter. Over the final 5,000 updates, the mean
raw loss is $6770.1$ for M2S and $6901.9$ for SEDD. This curve is an
optimization diagnostic, distinct from the held-out BPD estimates above.

\begin{figure}[H]
  \centering
  \includegraphics[width=0.96\linewidth]{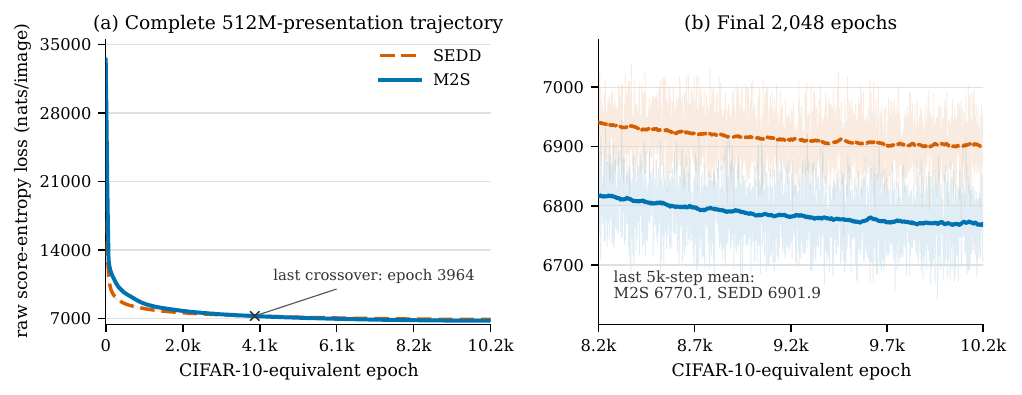}
  \caption{\textbf{M2S converges to a lower CIFAR-10 training
  objective.} Panel~(a) includes every optimizer step in the
  512M-image-presentation run; panel~(b) enlarges the final 2,048 equivalent
  epochs. Faint curves show raw per-step losses (subsampled only for rendering),
  and bold curves show trailing 625-step means, a 51.2-equivalent-epoch
  window. Both axes use equivalent epoch rather than wall-clock time or
  dataloader passes.}
  \label{fig:cifar10_training_loss}
\end{figure}

\paragraph{Generation protocol.}
We compare the EMA weights at update 125,000 for M2S and
SEDD. Both checkpoints have 28,427,520 parameters and were trained without
augmentation using the same optimizer, schedule, global batch size, and 512M
image-presentation budget. For each method, we generate 50,000 images with the
same 256-step categorical Euler discretization, cosine time grid, eight-GPU
rank assignment, per-GPU batch size 256, and seed 20260723. Every rank-local batch is
independently seeded from its first global sample index, so corresponding M2S
and SEDD images use the same initial noise and categorical random-number
stream even after a restart. We do not apply a final clean-token projection,
rejection, or sample filtering. The generation paths differ only in the learned
model output and the corresponding reverse-rate construction: M2S maps a
posterior through the bridge, whereas SEDD exponentiates a directly predicted
log-score.

\begin{table}[H]
\centering
\small
\setlength{\tabcolsep}{7pt}
\caption{\textbf{M2S improves both likelihood and sample quality in the controlled
CIFAR-10 comparison.} BPD is the four-pass variational upper bound on all
10,000 test images. FID-50k uses \texttt{pytorch-fid} v0.3.0 with
2,048-dimensional Inception pool-3 features against all 50,000 CIFAR-10
training images. Lower is better.}
\label{tab:cifar10_matched_fid}
\begin{tabular}{@{}llcc@{}}
\toprule
Method & Admissible score set & Test BPD ($\downarrow$) & FID-50k ($\downarrow$) \\
\midrule
M2S (ours) & $\mathcal P_{t,k}=B_{t,k}(\Delta^{K-1})$
  & $\mathbf{\leq 3.129}$ & $\mathbf{\CifarMtwoSFID}$ \\
SEDD & $\R_{>0}^{K-1}$
  & $\leq 3.173$ & $\CifarSEDDFID$ \\
\bottomrule
\end{tabular}
\end{table}

\begin{figure}[H]
\centering
\begin{subfigure}[t]{0.485\textwidth}
    \centering
    \includegraphics[width=\linewidth]{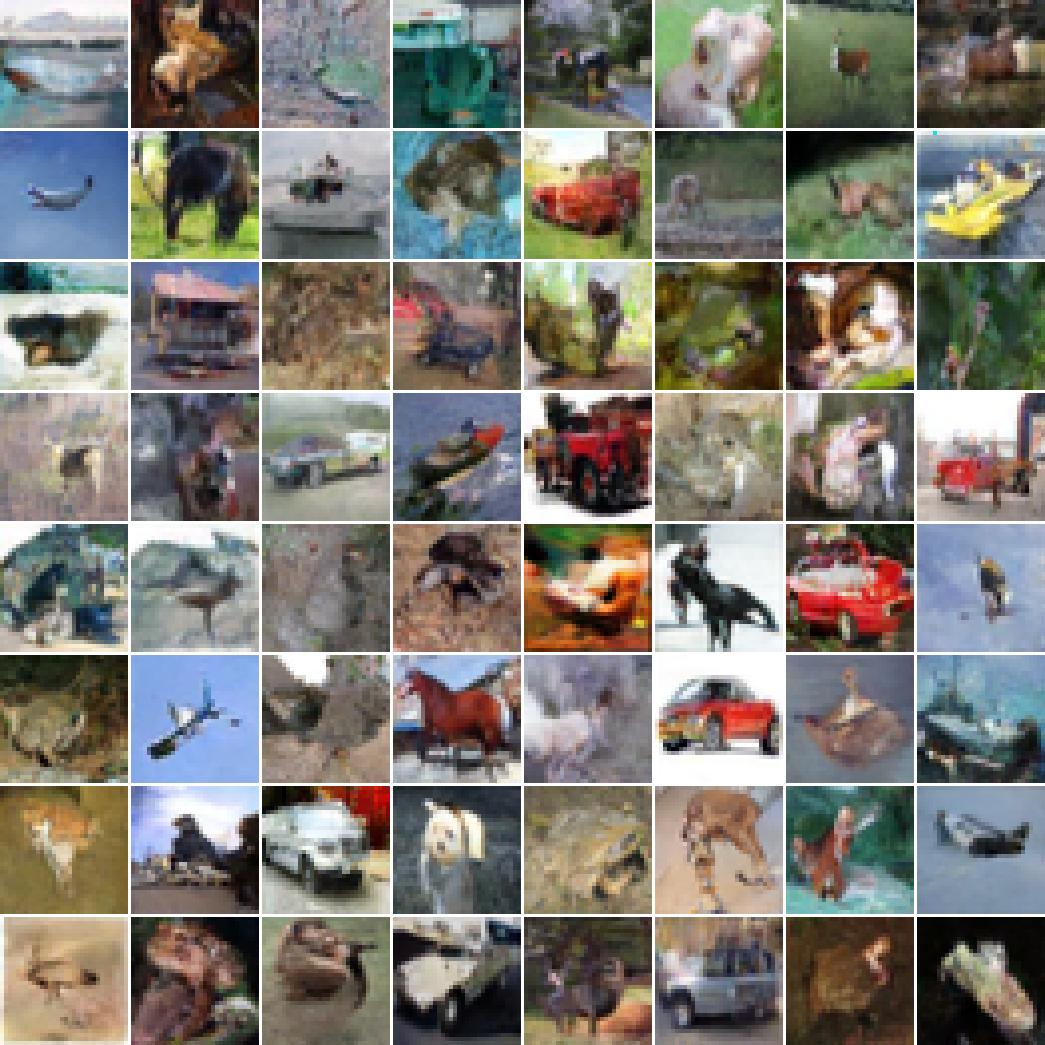}
    \caption{M2S, FID-50k $\CifarMtwoSFID$}
\end{subfigure}
\hfill
\begin{subfigure}[t]{0.485\textwidth}
    \centering
    \includegraphics[width=\linewidth]{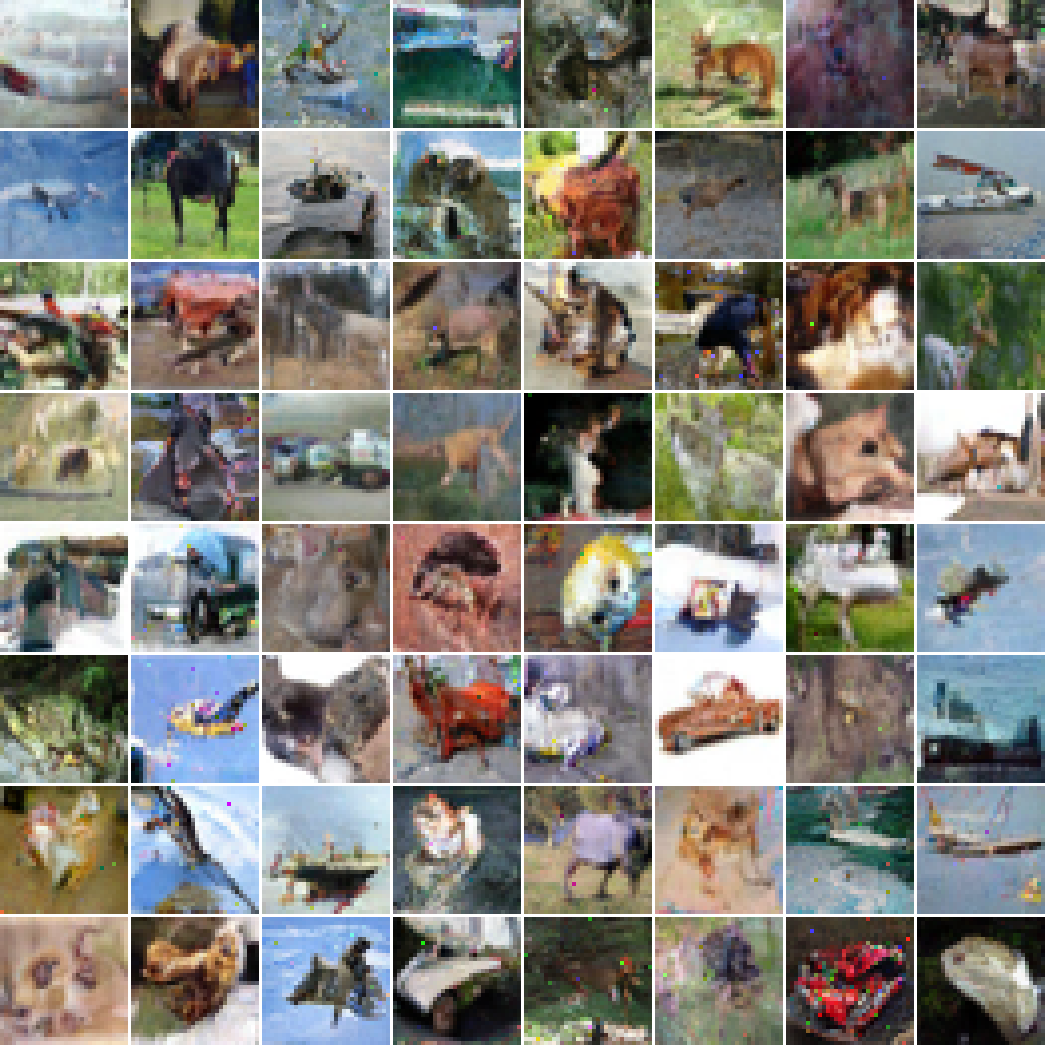}
    \caption{SEDD, FID-50k $\CifarSEDDFID$}
\end{subfigure}
\caption{\textbf{M2S suppresses the isolated chromatic artifacts visible in
SEDD under a fully paired sampler.} Each panel contains generated
indices 0--63 from its corresponding 50,000-image FID run, in index order,
without ranking, filtering, or manual selection. Corresponding positions use
the same seeded random-number stream, 256-step Euler discretization, and cosine
grid.}
\label{fig:cifar10_matched_samples}
\end{figure}

\paragraph{Connection to Bayes realizability.}
Theorem~\ref{thm:bridge} shows that the true score vector is induced by one
clean-token posterior. Because M2S predicts that posterior in the simplex, its
entire off-diagonal score vector lies in the bridge polytope
$\mathcal P_{t,k}$ from Eq.~\eqref{eq:score_realizability_sets} at every network
iterate. SEDD enforces positivity coordinate by coordinate but does not
enforce this joint constraint. Under finite data, capacity, and optimization,
the M2S parameterization therefore removes non-Bayesian score degrees of
freedom and promotes a coherent reverse-rate field. In the controlled experiment,
this structural restriction coincides with a $\CifarFIDDelta$-point FID
reduction ($\CifarFIDReductionPct$), the lower test BPD in
Table~\ref{tab:cifar10_matched_fid}, and fewer
isolated high-saturation pixels and local texture breaks in
Figure~\ref{fig:cifar10_matched_samples}. Both parameterizations contain the
population-optimal score, so the theorem alone does not imply a universal FID
ordering; the paired result is empirical evidence consistent with the proposed
finite-training mechanism.

\subsection{OpenWebText Experiments}
\label{app:owt_experiments}

\paragraph{M2S training.}
We train a 169.7M-parameter diffusion Transformer on OpenWebText with
GPT-2~\citep{radford2019language} tokenization and a maximum sequence length
of 512. We use the uniform forward
process and the M2S loss in Eq.~\eqref{eq:m2s_objective}, setting the
time-sampling cutoff to $10^{-3}$. The global batch size is 1584 across 24 H100
GPUs. We use AdamW with a peak learning rate of $5\times10^{-4}$, 3200 linear
warmup steps, cosine decay, gradient clipping at $1.0$, and bf16 precision. The
model is trained on 262B tokens.

\begin{figure}[H]
  \centering
  \includegraphics[width=0.80\linewidth]{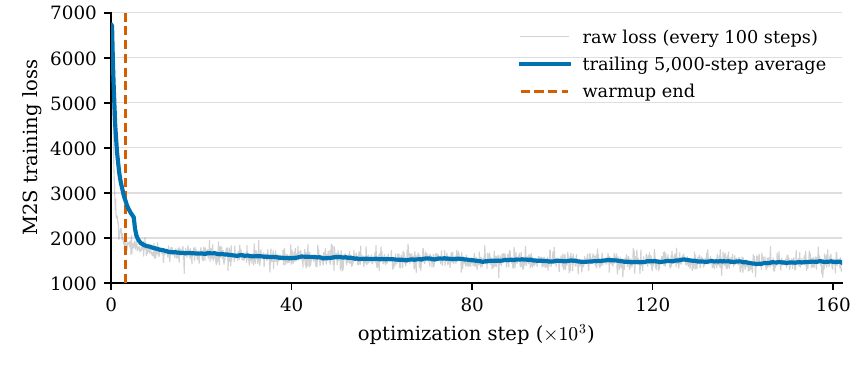}
  \caption{\textbf{M2S training loss over 161.8k optimization steps.} The gray
  curve shows raw loss recorded every 100 steps, the blue curve shows its
  trailing 5,000-step average, and the dashed line marks the end of warmup.}
  \label{fig:m2s_owt_training_loss}
\end{figure}

\paragraph{GenPPL evaluation.}
For each method and sampling-step budget, we draw 1024 unconditional samples
and score 512 tokens per sample with the same Gemma2-9B evaluator. For
checkpoints with \texttt{max\_len}=1024, we evaluate only the first 512
generated tokens; all other checkpoints generate 512-token sequences. We
report GenPPL as the exponentiated average next-token NLL assigned by the
evaluator. To evaluate the baselines with
their intended sampling procedures, MDLM, GIDD, SEDD, and Neural CTMC each use
the sampler specified in the corresponding paper~\citep{sahoo2024simple,
vonrutte2025gidd,lou2024discrete,li2026neuralctmc}. The sampling algorithms
used for M2S are described in Appendix~\ref{app:training_algorithm}.
Generated samples are shown in Appendix~\ref{app:owt_samples}.

\section{SEDD Bayes Realizability: Projection Repair, Violation Rates, and Sample Quality}
\label{app:posthoc_projection}

This appendix develops and evaluates projection repair for SEDD scores. We
first distinguish affine reconstruction, which reveals whether a score has a
valid clean-token posterior, from simplex-constrained projection, which maps an
non-realizable score to the nearest Bayes-realizable one. We then measure
Bayes-realizability
violations during SEDD sampling. Finally, we compare sample quality before and
after projection repair at a fixed checkpoint and describe the projected SEDD
sampler.

\subsection{Projection Repair}

For current token $k$, time $t$, and an SEDD candidate-score vector $s$,
we seek the valid clean-token posterior whose induced score is closest to $s$:
\begin{equation*}
  \mu^{\mathrm{proj}}(s)
  =\operatorname*{arg\,min}_{\mu\in\simplex}
  \lVert B_{t,k}\mu-s\rVert_2^2,
  \qquad
  s^{\mathrm{proj}}(s)=B_{t,k}\mu^{\mathrm{proj}}(s).
\end{equation*}
Thus $s^{\mathrm{proj}}(s)$ is the Euclidean projection of $s$ onto the set
of scores produced by valid posteriors,
$\mathcal P_{t,k}=B_{t,k}(\simplex)$. Euclidean projection onto the probability
simplex is a standard optimization primitive
\citep{duchi2008efficient,condat2016fast}. Our objective measures distance after
the linear score map $B_{t,k}$, so Section~\ref{app:projection_solver} derives a
specialized solver for the uniform bridge.

The simplex constraint is what makes this projection nontrivial. The matrix
$B_{t,k}\in\R^{(K-1)\times K}$ has full row rank, so if $\mu$ were allowed to
be an arbitrary real vector, every $s\in\R^{K-1}$ could be written exactly as
$B_{t,k}\mu$ and the projection would leave $s$ unchanged. Even after imposing
$\mathbf 1^\top\widetilde\mu=1$, every $s$ has a unique affine inverse, but
this inverse may contain negative entries. Requiring
$\widetilde\mu\in\simplex$, including nonnegativity, is what restricts the
score to $\mathcal P_{t,k}$.

For the uniform bridge, this affine inverse has a closed form. Let $n=K-1$,
$d_t=1-\rho_t$, and $b_t=d_t/\rho_t$. Then
\begin{equation*}
  m(s)
  =\frac{\sum_{y\ne k}(s_y-\rho_t)}
  {d_t(n+\rho_t^{-1})},
  \qquad
  \widetilde\mu_k=1-m(s),
  \qquad
  \widetilde\mu_y
  =\frac{s_y-\rho_t-d_t m(s)}{b_t}.
\end{equation*}
It is important to distinguish reconstruction from projection. The
SEDD head produces
\begin{equation*}
  s_y^{\mathrm{SEDD}}=\exp(a_{\theta,y}),\qquad y\ne k.
\end{equation*}
Applying the affine inverse above gives a unique signed vector
$\widetilde\mu(s^{\mathrm{SEDD}})$ with
$\mathbf 1^\top\widetilde\mu=1$. Reapplying the bridge defines the
reconstruction
\begin{equation*}
  s^{\mathrm{rec}}
  :=B_{t,k}\widetilde\mu(s^{\mathrm{SEDD}})
  =s^{\mathrm{SEDD}}.
\end{equation*}
The last equality is an algebraic identity on the full affine hyperplane, so
$s^{\mathrm{rec}}$ is a numerical consistency check, not a repair. In
particular, a negative entry of $\widetilde\mu$ can coexist with
$s^{\mathrm{rec}}=s^{\mathrm{SEDD}}$.

Bayes realizability requires $\widetilde\mu\in\simplex$. Consequently, only
the simplex-constrained projection changes a score outside the bridge
polytope:
\begin{equation*}
  \mu^{\mathrm{proj}}
  :=\mu^{\mathrm{proj}}(s^{\mathrm{SEDD}}),
  \qquad
  s^{\mathrm{proj}}:=B_{t,k}\mu^{\mathrm{proj}}.
\end{equation*}
Thus $s^{\mathrm{proj}}=s^{\mathrm{SEDD}}$ if and only if
$\widetilde\mu\in\simplex$.

Section~\ref{app:projection_solver} derives the scalar-threshold solver used by
the projected SEDD sampler.

\subsection{Bayes-Realizability Violation Rates}

We use three sets of samples for different purposes: four trajectories for the
diagnostic in Figure~\ref{fig:owt_checkpoint_audit}(b), 128 trajectories to
measure Bayes-realizability violations, and 1,024 paired sequences to evaluate
projection repair. Because the three analyses use different units, their
reported percentages should be interpreted separately.

\begin{table}[H]
\centering
\caption{Data used in the three SEDD score analyses. The first two rows inspect
SEDD scores while leaving the generated trajectories unchanged. The third
applies projection repair during sampling.}
\label{tab:sedd_projection_cohorts}
\small
\setlength{\tabcolsep}{4pt}
\begin{tabularx}{0.98\linewidth}{@{}
  >{\raggedright\arraybackslash}p{0.20\linewidth}
  >{\raggedright\arraybackslash}X
  >{\raggedright\arraybackslash}X@{}}
\toprule
Data & Purpose & Reported unit \\
\midrule
4 SEDD trajectories
  & Negative-weight diagnostic in Figure~\ref{fig:owt_checkpoint_audit}(b)
  & Step-averaged candidate-score and position rates \\
128 SEDD trajectories
  & Bayes-realizability violation rates
  & $8{,}388{,}608$ position-step score vectors \\
1,024 paired sequences
  & Sample quality before and after projection repair
  & $67{,}633{,}152$ position updates and sequence quality \\
\bottomrule
\end{tabularx}
\end{table}

For the first two analyses, we used the original SEDD scores throughout
sampling. At each state and position, we recorded the complete score vector and
checked whether every coordinate lay in the required interval and whether the
complete vector was induced by a valid clean-token posterior. These checks did
not alter the generated trajectories. Only the projection experiment in
Section~\ref{app:sedd_projection_intervention} replaced
$s^{\mathrm{SEDD}}$ with $s^{\mathrm{proj}}$ during sampling.

\paragraph{Sampler-weight violations.}
We ran four length-512 sequences with the 128-step SEDD sampler. At each step
we evaluated the SEDD score $s^{\mathrm{SEDD}}$ and its projection on the same
current states. Averaged over steps, $9.781\%$ of off-diagonal candidate scores
violated the coordinate box, and $37.632\%$ of positions had at least one
negative pre-normalization weight. Projection removed all observed negative
weights, giving the diagnostic in Figure~\ref{fig:owt_checkpoint_audit}(b).

For the 128-trajectory analysis, we used the partition in
Eq.~\eqref{eq:score_realizability_partition}. To distinguish material
violations from numerical effects near the boundary, we split
$\mathcal C_{t,k}\setminus\mathcal P_{t,k}$ at
$\varepsilon_\mu=10^{-6}$. Table~\ref{tab:realizability_taxonomy} reports the
resulting four classes over 512 positions and 128 sampler steps per trajectory.

\begin{table}[H]
\centering
\caption{SEDD score vectors over 128 trajectories, using
$\mathcal P_{t,k}=B_{t,k}(\simplex)\subsetneq\mathcal C_{t,k}
\subsetneq\R_{>0}^{K-1}$. The boundary row separates numerical from material
violations.}
\label{tab:realizability_taxonomy}
\small
\setlength{\tabcolsep}{4pt}
\begin{tabularx}{0.96\linewidth}{@{}lXrr@{}}
\toprule
Position-step class & Criterion & Fraction & Count \\
\midrule
Outside $\mathcal C_{t,k}$
  & $s^{\mathrm{SEDD}}\notin\mathcal C_{t,k}$
  & $25.0177\%$ & $2{,}098{,}635$ \\
In $\mathcal C_{t,k}\setminus\mathcal P_{t,k}$, material
  & $s^{\mathrm{SEDD}}\in\mathcal C_{t,k}\setminus\mathcal P_{t,k}$ and
    $\min\widetilde\mu<-10^{-6}$
  & $56.2310\%$ & $4{,}716{,}994$ \\
In $\mathcal C_{t,k}\setminus\mathcal P_{t,k}$, boundary
  & $s^{\mathrm{SEDD}}\in\mathcal C_{t,k}\setminus\mathcal P_{t,k}$ and
    $-10^{-6}\le\min\widetilde\mu<0$
  & $18.7419\%$ & $1{,}572{,}188$ \\
In $\mathcal P_{t,k}$
  & $s^{\mathrm{SEDD}}\in\mathcal P_{t,k}$
  & $0.0094\%$ & $791$ \\
\bottomrule
\end{tabularx}
\end{table}

Across position-steps, $25.0177\%$ lay outside $\mathcal C_{t,k}$, $56.2310\%$
were material violations in
$\mathcal C_{t,k}\setminus\mathcal P_{t,k}$, $18.7419\%$ formed its numerical
boundary band, and only $0.0094\%$ lay in $\mathcal P_{t,k}$. The $56.2310\%$
rate was counted directly because the marginal coordinate and posterior-sign
violations overlap ($1.2584\%$).

\subsection{Sample Quality after Projection Repair}
\label{app:sedd_projection_intervention}

To isolate the effect of projection repair, we kept the pure-uniform SEDD checkpoint
and sampler fixed and changed only the score passed to the sampler. It received
either $s^{\mathrm{SEDD}}$ or
$s^{\mathrm{proj}}=B_{t,k}\mu^{\mathrm{proj}}$. For each setting, we generated
1,024 sequences of length 512 with 128 sampling steps. Each pair started from
the same initial noise and used the same random numbers and final denoising
update.

With projected scores, none of the observed pre-normalization sampler weights
was negative. We decoded the samples with the GPT-2 tokenizer and measured
external GenPPL using the same Gemma2-9B evaluator as in the main language
experiment. Table~\ref{tab:sedd_projection_ablation} summarizes the comparison;
confidence intervals are estimated by sequence bootstrap, with samples paired
across the two settings.

\begin{table}[H]
\centering
\caption{Effect of projection repair at a fixed SEDD checkpoint. The checkpoint
and sampler are identical in both rows; only the score input changes. Intervals
are 95\% sequence-bootstrap confidence intervals.}
\label{tab:sedd_projection_ablation}
\small
\setlength{\tabcolsep}{5pt}
\begin{tabular}{@{}lccc@{}}
\toprule
Score input & Average NLL & External GenPPL & 95\% GenPPL interval \\
\midrule
$s^{\mathrm{SEDD}}$ & $5.31615$ & $203.60$ & $[196.64,210.45]$ \\
$s^{\mathrm{proj}}=B_{t,k}\mu^{\mathrm{proj}}$
  & $5.16518$ & $175.07$ & $[169.10,181.19]$ \\
\bottomrule
\end{tabular}
\end{table}

Using projected scores lowered external GenPPL from $203.60$ to $175.07$, a
$14.0\%$ reduction. The paired bootstrap consistently favored projection, and
all generated sequences were distinct in both settings, so the gain was not
explained by sequence collapse.

Because the checkpoint and sampler were unchanged, this comparison isolates
the effect of projection repair at inference time for this
checkpoint. It does not imply that post-hoc projection is equivalent to
training M2S; external GenPPL is also an external measure of sample quality,
not the model's likelihood.

\subsection{Projected SEDD Sampler}
\label{app:projection_solver}

The uniform bridge permits an exact reduction of the simplex-constrained
projection to a scalar threshold. Let $u=\mu_{-k}$,
$z=s^{\mathrm{SEDD}}_{-k}-\rho_t\mathbf 1$, $n=K-1$,
$d=1-\rho_t$, and $b=d/\rho_t$. Since $\mu_k=1-\mathbf 1^\top u$,
Eq.~\eqref{eq:uniform_bridge} gives
\begin{equation}
  \min_{u\ge 0,\ \mathbf 1^\top u\le 1}
  \left\|(bI+d\mathbf 1\mathbf 1^\top)u-z\right\|_2^2.
  \label{eq:uniform_score_projection_reduction}
\end{equation}
For the objective in Eq.~\eqref{eq:uniform_score_projection_reduction}, define
\begin{equation*}
  c=b^2,
  \qquad
  h=d^2\left(n+\frac{2}{\rho_t}\right),
  \qquad
  g_y=bz_y+d\sum_{v\ne k}z_v.
\end{equation*}
The KKT conditions give
\begin{equation*}
  u_y=\frac{[g_y-\theta]_+}{c}.
\end{equation*}
If the mass constraint is inactive, the scalar threshold satisfies
$\sum_{y\ne k}u_y=\theta/h$ with $\theta\in[0,h]$. If it is active, it
satisfies $\sum_{y\ne k}u_y=1$ with
$\theta\in[h,\max_{y\ne k}g_y]$. At the branch point $\theta=h$, let
$m_h=c^{-1}\sum_{y\ne k}[g_y-h]_+$. The mass constraint is active exactly
when $m_h>1$. In either branch, the mass on the left decreases monotonically
while the target on the right is nondecreasing, so batched bisection solves all
positions without sorting the vocabulary. Here $[a]_+=\max\{a,0\}$. The
resulting projection is applied at every position before the standard SEDD
update, as summarized in Algorithm~\ref{alg:projected_sedd_sampler}.

\begin{algorithm}[H]
\caption{Projected SEDD Sampler}
\label{alg:projected_sedd_sampler}
\begin{algorithmic}[1]
\Require trained SEDD model, reverse-time grid
$t_0>\cdots>t_M$, terminal distribution
\State Sample $X_{t_0}$ from the terminal distribution
\For{$m=0,\ldots,M-1$}
  \State $x\gets X_{t_m}$
  \State Evaluate the SEDD score $s_i^{\mathrm{SEDD}}$ for every position $i$
  \ForAll{positions $i$ \textbf{in parallel}}
    \State $k\gets x^i$ and
    $\mu_i^{\mathrm{proj}}\gets\mu^{\mathrm{proj}}(s_i^{\mathrm{SEDD}})$
    using Eq.~\eqref{eq:projected_posterior}
    \State $s_i^{\mathrm{proj}}\gets B_{t_m,k}\mu_i^{\mathrm{proj}}$
  \EndFor
  \State Apply the standard SEDD update with $s^{\mathrm{proj}}$ to sample
  $X_{t_{m+1}}$
\EndFor
\State Apply the standard final denoising update
\State \Return $X_0$
\end{algorithmic}
\end{algorithm}

\section{Sample Generation}
\label{app:owt_samples}

The following passages are randomly selected excerpts from samples generated
by the M2S model using 128 sampling steps.

\begin{owtsample}{Sample 1}
The European Union (EU) said the European Union would ``fight'' Russia's
aggression against ISIS with ``love for neighbour'' and achieve a ``defaulted''
nuclear agreement and prevent the launch of a nuclear weapon.

Among the top pro-Moscow-Kremlin claims about the Minsk peace agreement were
that in 2014, however the war a far east pro-Russian separatists had stopped
their long and violent clashes with Russia.

That was still false now, the EU leaders said in a joint statement delivered by
German Chancellor Angela Merkel late on Monday. They claimed that ceasefire
negotiations would continue after 2014, but added that their ``government has to
compromise,'' before rushing to Minsk.

This came during a brief discussion with President Merkel at a joint news
conference of the Joint Forces Committee of the Russian Aerospace Forces.
\end{owtsample}

\begin{owtsample}{Sample 2}
A Republican proposal to change Senate rules would force the Senate budget in
line for 10 percent, but under current law the regs spending is \$77 million.

The budget proposal submitted on Tuesday by Sen. Marco Rubio, R-Fla., the
Senate's third stiffer member and didn't include the first shortfall, a
Republican said, one that he now faces would cost a Republican incumbent
\$3,000 in the U.S. Senate, the unnamed Republican said.

While Senate Democrats were unavailable for comment about the Rubio proposal,
the Senate's budget from 2009 to fiscal 2013 was \$120 million, with Senate
Majority Leader Harry Reid, D-Nev., calling for savings to narrow the gap.

Whitehouse communications director Matt Lunt responded, saying that while any
new bill would be pork, full funding would have been required, if the usual
funds made financial sense.

Leavint Reid, D-Nev., had not responded to questions about the proposal.

According to the Senate, Rubio would have requested \$14 million---a sum that
became somewhat problematic in 2011. In the final 15 days that he had given
into campaign finance, he closed a public college and policy research board,
which had no students. The board of education was appointed by Obama.
\end{owtsample}

\begin{owtsample}{Sample 3}
Apple has revealed a new 5K tablet with a redesigned, bezel-less Touch Cover
design. It's a bit cumbersome and a little more of a departure than the first
redesign will be expected. It looks harmless, especially considering launching
the invitation ring on the new 5800 later this year, with another refreshed
version soon.

What we should see in the huge next iPhone 5S is a tablet stand, costing
\$420,000--\$850,000, along with a full top shelf smartphone. Such mini-boxes
will serve them the power down and thus the memories, and possibly drive ports
are all possible.

You'd most likely wish they'd wrap this thing up, rather than with their own
logo but at least that's what we can expect. Apple will also offer the keyboard
that pairs with it like you're forced to the top of its sides. Obviously we
find that as a nice gesture and you'd think that wouldn't have to be available
during launch too, but if you take the \$300--\$400, the setup will fit very
nicely.

Apple has also prepared a new iPad, just so you can have it with it too, but
this idea is a myth. Not bad, since the iPhone offers a 4K resolution to 4k
video and up to 12GB should be very respectable, but its not why the phone
comes with Android 4.1/2.0 OS. Stay tuned for more hands-on details about the
latest, and we'll report back to iPad TV Line in the future for more info.
\end{owtsample}

\begin{owtsample}{Sample 4}
Junior Robinson is a genius. He's always been blessed with natural size and
speed. Lemon speed is about any player. But the speed he cultivated in the
offseason, and he has averaged just six minutes in practice ice in a season
starting the past three is at all impressive.

And it's not something you can't. People can carry a a couple of minutes in
five preseason games, able to attend a freeze from assembling the NHL roster.
They are versatile, guys.

The 24-year-old former 18th overall selection chose Detroit as the first-round
pick of the 2012 NHL Pending Draft, two years after averaging 25 points in 24
games in various stints at 82 points, and starting since this past October, for
Detroit, in the same spot some of his most promising stints in short-re 30
games, and 24 points in 27, for a rookie season, he performed very well.

He returned this offseason, maximizing his own potential by one season at the
NHL---subject to arbitration until 2012 at his due date. In order for the Red
Wings to make use of roster room, they need just a couple of 25-man
reinforcements in their rotation, without shipping John O'Neil or others for a
roster spot.
\end{owtsample}

\begin{owtsample}{Sample 5}
The year is once again in 2015, as the number of layoffs has been a record year
high. Consequently, the banking giant Chase \& Chase Bank recently reported
that the sector had lost more than 19,000 jobs over the past three years.

Read more

Markers, who are criticised for feeling a stagnation, usually cite strong
employment growth trends to speak to its current trends. ``Today, the
unemployment rate is down by 330 per cent since 2010,'' said in its annual
report. ``And that depends on the shifts in opportunities in business.''

However, in 2014--2015, one Fargo Ontario office employing 17,000 from was
downsized---by 5,200---it backtracked this. It eliminated more than two-thirds
of production, filling out the need among its banking sector employees.

Earlier, the Bank had hinted that it may tighten its finances more. It released
evidence that a low in wage workers caused the unemployment cut by almost
\$14,000 to workers before a tax of 0.25 percent. This would push the GDP worse
in 2015. The recovery could not have been struck without these defects.

As Fargo is expected to come out soon about the disinitation of its economy,
the report, however, states the need to change infrastructure in a more
sustainable way.

The sale of Fargo operating Business Hummer division was taken in as the
expense by example cited above.
\end{owtsample}

\begin{owtsample}{Sample 6}
The main web apps on the desktop are 8GB allowed KB. They have a 40\% download
speed, and get a chunk of that from the Google Play Store, you can bet that
Google also intends to bring the fun back into the store by having the same
much functionality as in Firefox browsers like Chrome and Safari.

On the mobile side, we have Hoothing, Magazines, Cycle Enhancement Browser,
and Goggles apps, though this technically means you can upgrade later this
year if you tried and missed before. Full apps can go on now too.

Apps are also available on self-link for third-party developers to help with
the Web apps as well as the Mobile settings also. However, such details were
not available at this stage of the above launch.

Chrome users, however, need to enable options even for Android. That's been
going on for many years now, with the Google Chrome desktop live app (which has
been live for nearly a year now) and the mobile browsers on the U.S. and
Chinese Web versions. Still enough stuff, you know?

Apps first launched in early September will now be supported including on
Google's website, therefore making the Google Chrome app easier to sign up.
There is still an open beta. All things are ready for everyone, always and
soon.

Google is testing out the redboards in China and other countries and is
testing the locations available on the site.
\end{owtsample}

\begin{owtsample}{Sample 7}
When you sign up for an AWS Elastic Load Network (ELT), get connected with
your email or even the online language in your applications, it gets you go
immediately. This may have only gotten so rare in the past, but the full power
of computing is up for grabs in cloud computing and with its acquisition by
PxCom, Inc. its cloud-delivered services has always been unusually designed
with one important one: customer delivery in mind.

``As the dominant broadband-only provider in the world, Seattle-based AWS has
particular plans to modernize its computing infrastructure and give it
real-time on-demand service and production-ready broadband customer delivery,
with direct access to its delivery systems,'' its agency says in a release.

It also understands that it can be a layer used by its other large providers
and coordinate their services to their customers's data centers and connect
them in their case at home, an indication that not all of the customer
experiences in the cloud are necessarily different from each of the largest
ISPs' portfolio companies.

P\&Com, founded in partnership with Yormia and Escape Networks in January of
2014, is a provider of enterprise cloud service for ``cloud'' services AWS,
Microsoft Azure, Amazon Dropbox, and Google Docs.
\end{owtsample}

\begin{owtsample}{Sample 8}
Microsoft has announced it will announce and released of AI assistant in around
the holidays.

Last Thursday, Microsoft revealed that its AI assistant device, which was shown
as a prototype over at CES, ``is the successor to Google tablet- AI operating
system with a text and call interface.'' It will likely be paired with the
upcoming Microsoft-powered Xbox One.

Now there is another hugely interesting and interesting rumor floating around,
and it got execs caught at CES anyway. An official Windows blog post is
claiming that the new software will feature in users' mobile systems, and it
will provide a ``personal voice assistant for device hubs with limited sharing
space,'' Microsoft said. Unlike Siri, it works with older storage devices, but
there stored multiple SD cards or at least a volume on them. That's obviously
a waste of time, as your device can be gone out-of-state after being hooked up
to the SD card.

The AI device is all tied to Microsoft's tech future, the AI+ operating system.
Siri that will be capable of communicating with Internet-connected PCs, and the
goal will be to work with iOS, Mac, and other groups of devices in common,
require AI+ using operating system.

We said, but wait. The full video is available here. $\mu$
\end{owtsample}

\begin{owtsample}{Sample 9}
With a new law pending before Congress, California faces an uphill labor
battle, chief executive officer Xavier Becerra said recently.

There is an uproar over California's pay rate for the corporate behemoth that
sells health care, which makes home appliances and other products.

Western leaders on Capitol Hill still struggling with lower wages in countries
that lean to union workers and rely on pliers' compensation. But here were
signs of a change of heart that amounts to a higher rate for California,
partly because taxpayers in the state have invested heavily in company labor.

The federal labor department said rates in its factories in Tennessee were
also getting higher. ``As there has been such a measure of growth in the
states, labor costs have risen,'' Mr. Cosino said in ODFO's fact sheet March
20.
\end{owtsample}

\begin{owtsample}{Sample 10}
The International Maoist-China forum called in return for ``a world leader in
international stability and international goods and services''.

``China is proposing that its vision of state-denialing is a new machine,
based on international cooperation, and its own market, rather than a country
who would use its culture and opportunity to lead,'' Hu Bingkao told
participants, in an announcement of the convention on Saturday.

Saturday, as the party celebrated the first anniversary of the start of the
Maoist Commission, which was held in October when it warned of economic and
social unrest in the country.

The government, at the time felt there was no warning about the new trend,
which included the rise of crack economy, ultra-high unemployment rates,
levels of corruption, and the growth of militant groups.

The convention took place just north of the Helfen, where also collected in
some over 60 addresses take place annually.

``Our society will be part of the discussion, but our ideological reason is
simply the objective,'' the Bingkao said.

Meanwhile, Hu said People needed to grow the environment, dealing with drug
dealing, and said that the methods to boost social stability was needed to
respond to problems for the population.
\end{owtsample}

\end{document}